\documentclass[twoside]{amsart}
\usepackage[utf8]{inputenc}
\usepackage{amsmath}
\usepackage{amsfonts}
\usepackage{amssymb}
\usepackage{graphicx}
\usepackage{color}
\usepackage{setspace}
\usepackage[normalem]{ulem}
\usepackage{graphicx}
\usepackage{epstopdf}
\usepackage{float}
\usepackage{overpic}
\usepackage[left=2cm,right=2cm,top=2cm,bottom=2cm]{geometry}

\numberwithin{equation}{section}
\newtheorem{theorem}{Theorem}
\newtheorem{lemma}{Lemma}
\newtheorem{proposition}{Proposition}

\newtheorem{definition}{Definition}
\newtheorem{remark}{Remark}
\newtheorem{example}{Example}
\numberwithin{theorem}{section}
\numberwithin{lemma}{section}
\numberwithin{proposition}{section}
\numberwithin{corollary}{section}
\numberwithin{definition}{section}
\numberwithin{example}{section}
\numberwithin{remark}{section}

\def\NN{\mathbb N}

\def\BB{\mathbb B}
\def\RR{\mathbb R}
\def\GG{\mathbb G}

\def\SS{\mathbb S}

\def\EE{\mathbb E}
\def\TT{\mathbb T}
\def\XX{\mathbb X}
\def\YY{\mathbb Y}
\def\PP{\mathbb P}
\def\x{\mathbf{x}}
\def\y{\mathbf{y}}
\def\w{\mathbf{w}}
\def\by{\mathbf y}
\def\bz{\mathbf z}

\def\Ex{\mathcal{E}_x}
\def\Exx{\mathcal{E}_\x}
\def\mfc{\mathfrak{C}}

\def\be{\begin{equation}}
\def\ee{\end{equation}}
\def\bea{\begin{eqnarray}}
\def\eea{\end{eqnarray}}

\def\prob{\mbox{{\rm Prob }}}
\def\donchitre#1#2{\vskip 6.5cm\noindent
\parbox[t]{1in}{\special{eps:#1.eps x=6.5cm y=5.5cm}}
\hbox to 7cm{}\parbox[t]{0.0cm}{\special{eps:#2.eps x=6.5cm y=5.5cm}}}

\def\dist{\mathsf{dist }}
\def\supp{\mathsf{supp }}

\begin{document}
\title{Tractability of approximation by general shallow networks}

\author{Hrushikesh Mhaskar$^1$}
\thanks{
$^1$Institute of Mathematical Sciences, Claremont Graduate University, Claremont, CA 91711, U.S.A. The research of HNM was supported in part by NSF grant DMS 2012355, and ONR grants N00014-23-1-2394, N00014-23-1-2790.
\email{hrushikesh.mhaskar@cgu.edu}} 

\author{Tong Mao$^2$}
\thanks{
$^2$Institute of Mathematical Sciences, Claremont Graduate University, Claremont, CA 91711, U.S.A.  The research of this author was supported by NSF DMS grant 2012355.
\email{tong.mao@cgu.edu}}

\begin{abstract}
In this paper, we present a sharper version of the results in the paper Dimension independent bounds for general shallow networks; Neural Networks, \textbf{123} (2020), 142-152.
Let $\mathbb{X}$ and $\mathbb{Y}$ be compact metric spaces. We consider approximation of functions of the form $ x\mapsto\int_{\mathbb{Y}} G( x, y)d\tau( y)$, $ x\in\mathbb{X}$,  by $G$-networks of the form $ x\mapsto \sum_{k=1}^n a_kG( x, y_k)$, $ y_1,\cdots, y_n\in\mathbb{Y}$, $a_1,\cdots, a_n\in\mathbb{R}$.
Defining the dimensions of $\mathbb{X}$ and $\mathbb{Y}$ in terms of covering numbers, we obtain dimension independent bounds on the degree of approximation in terms of $n$, where also the constants involved are all dependent at most polynomially on the dimensions.
Applications include approximation by power rectified linear unit networks, zonal function networks, certain radial basis function networks as well as the important problem of function extension to higher dimensional spaces. 
\end{abstract}

\maketitle

\section{Introduction}\label{sec:intro}
An important problem in the study of approximation of functions of a large number of input variables is the curse of dimensionality. 
For example, in order to get an accuracy of $\epsilon>0$ in the approximation of a function that is $r$ times continuously differentiable on $B^q$, based on continuous parameter selection (such as values of the function or initial coefficients in an appropriate orthogonal polynomial expansion) the number of parameters required is at least a constant multiple of $\epsilon^{-q/r}$, a quantity that tends to infinity exponentially fast in terms of the input dimension \cite{devore1989optimal}.

Naturally, there are several efforts to mitigate this curse.
We mention two of these.

One idea is to assume the so called \emph{manifold hypothesis}; i.e., to assume that the nominally high dimensional input data is actually  sampled from a probability distribution supported on a low dimensional sub-manifold of the  high dimensional ambient space. 
The theory of function approximation in this context is  very well developed \cite{mauropap, eignet, heatkernframe, compbio, modlpmz, mhaskar2020kernel,mhas_sergei_maryke_diabetes2017}.
In practice, this approach has been used successfully in the context of semi-supervised learning; i.e., in the case when all the data is available to start with but only a small number of labels are known \cite{belkin2004semi,belkin2004regularization,gavish2010multiscale}. 
However, when the function needs to be evaluated at a new data point, the entire computation needs to restart. 
This is referred to as the problem of \emph{out of sample extension}.
Nystr\"om extension is often made to solve this problem (e.g., \cite{zhang2016sampling,belkin2006convergence,data_based_construction2019}), but then one does not have any approximation guarantees on the processes so extended.

The other recent idea is to use deep networks.
We have observed in \cite{dingxuanpap} that deep networks help to mitigate the curse of dimensionality when the target functions have compositional structures. Unlike shallow networks, deep network architectures can reflect these structures.
For example, if one wants to approximate
$$
f(x_1,x_2,x_3,x_4)=f_1(f_{11}(x_1, x_2), f_{12}(x_3,x_4)),
$$
the compositional structure ensures that each channel in a deep network of the form
$$
P(x_1,x_2,x_3,x_4)=P_1(P_{11}(x_1, x_2), P_{12}(x_3,x_4)),
$$
where $P_1, P_{11}, P_{12}$ are suitably constructed neural networks, is working only with bivariate functions rather than a function of four variables. 
Therefore, the number of parameters required to get an accuracy of $\epsilon$ is only $\mathcal{O}(\epsilon^{-2/r})$ rather than $\epsilon^{-4/r}$.
A major consequence of this observation is that the approximation theory for deep networks is reduced to that for shallow networks.
A second consequence is that deep networks will not perform better if the functions involved  belong to classes which do not have a curse of dimensionality to begin with.
In \cite{mhaskar2020dimension,sphrelu}, we have proved drastically different bounds on the degree of approximation by neural networks evaluating the activation function of the form $|\circ|^\gamma$ ($\gamma$ not an even integer), depending upon whether the approximation is constructive or not. 
In \cite{sphrelu}, we have given explicit constructions with many desirable properties, such as weight sharing, rotation equivariance, stability, etc., but the degree of approximation suffers from what looks like a curse of dimensionality \cite{yarotsky2018optimal}.
For the \textbf{same class of functions}, \cite{mhaskar2020dimension} gives dimension independent but non-constructive bounds,  which are close to the optimal bound \cite{Barron2018}. 
So, the use of degree of approximation by itself without any reference to how the approximation is constructed is useless to determine the architecture and size of a neural network or kernel based machine, etc. which can then be trained using empirical risk minimization.

Nevertheless, it is an interesting theoretical question with a long history to investigate function classes which do not exhibit a curse of dimensionality when the approximation is not required to be based on a continuous parameter selection. 
A major class with this property is motivated by the following considerations.
A neural network with $N$ neurons using the so called ReLU$^\gamma$ activation function, $t\mapsto t_+^\gamma$ for some $\gamma>0$ has the form $\sum_{k=1}^N a_k( \x\cdot \y_k+b_k)_+^\gamma$, where $ \x, \y_k\in\RR^q$, $a_k,b_k\in\RR$, $k=1,\ldots,N$. 
By dimension-raising, i.e., writing $ \x'=( \x,1)/\sqrt{| \x|^2+1}$, $\w_k=( \y_k,b_k)/\sqrt{| \y_k|^2+b_k^2}$, $A_k=a_k(| \y_k|^2+b_k^2)^{\gamma/2}$, the network becomes $(x'_{q+1})^{-\gamma/2}\sum_k A_k( \x'\cdot\w_k)_+^\gamma$, where $A_k\in\RR$, $ \x', \w_k$ are on the $q$-dimensional sphere $\SS^q$ embedded in $\RR^{q+1}$. 
This can be expressed in an integral notation as $\int_{\SS^q} G( \x',\w)d\tau_N(\w)$ where $\tau_N$ is a discrete measure associating the mass $A_k$ with each $\w_k$, and $G( \x',\w)= (x'_{q+1})^{-\gamma/2}( \x\cdot\w)_+^\gamma$.
In the literature, it is usually assumed that the total variation of $\tau_N$ is bounded independently of $N$; e.g., \cite{Barron2018,siegel2020approximation,ma2022uniform,siegel2022high,pisier1981remarques,kurkova1, kurkova2}. Clearly, the only functions which can be approximated by such networks have the form $\int_{\SS^q} G( \x',\w)d\tau(\w)$ for some (signed) measure $\tau$ on $\SS^q$.
Similarly, with the canonical embedding of a reproducing kernel Hilbert space (RKHS) into the parent $L^2$ space given in \cite{aronszajn1950theory}, one can describe the RKHS as the space of functions of the form $\int G( x, y)\mathcal{D}_G(f)( y)d\mu( y)$ for some $\mathcal{D}_G(f)\in L^2(\mu)$.
When $G$ is a radial basis function, such spaces are often called the \emph{native spaces} for $G$ (e.g., \cite{schaback1999native, le2006continuous,hangelbroek2012polyharmonic,feng2023radial}). 
More generally, reproducing kernel Banach spaces can be described as a set of functions of the form $\int G( x, y)d\tau( y)$, where $\tau$ is a (signed) measure \cite{bartolucci2021understanding,song2013reproducing,song2011reproducing}.

Accordingly, we pause in our discussion to introduce a terminology to define the class of all such functions of interest in this paper.
\begin{definition}
    \label{def:variationspace}
Let $\XX$ and $\YY$ be metric measure spaces. A function $G:\XX\times \YY\to \RR$ will be called a \textbf{kernel}.
The \textbf{variation space (generated by $G$)}, denoted by $\mathcal{V}(G)$, is the set of all functions of the form
$ x\mapsto \int_\YY G( x, y)d\tau( y)$ for some signed measure $\tau$ on $\YY$ whenever the integral is well defined.
For integer $N\ge 1$, the set $\mathcal{V}_N(G)$ is defined by
\be\label{eq:manifolddef}
\mathcal{V}_N(G)=\left\{\sum_{k=1}^N a_k G( \circ,  y_k) : a_1,\cdots, a_N\in\RR,\  y_1,\cdots, y_N\in \YY\right\}.
\ee
An element of $\mathcal{V}_N(G)$ will be called a \textbf{$G$-network (with $N$ neurons)}.
\end{definition}

In much of the literature, approximation of functions in variation spaces is studied when $\XX$ and $\YY$ are special subsets of a Euclidean space (See Section~\ref{sec:relatedwork} for a brief discussion). In \cite{mhaskar2020dimension}, we had studied in an abstract setting the approximation of functions in $\mathcal{V}(G)$, where $G$ is a kernel defined on general metric spaces.
Under certain conditions, we obtained dimension independent bounds for approximation of such functions by shallow networks of the form $ x\mapsto\sum_{k=1}^N a_k G( x, y_k)$. 
In the special case when $\YY\subset \XX$, this allows us to obtain approximation results in the \emph{out of sample} case as well.
A novelty of this work is that the smoothness of the kernels as well as bounds on certain coverings of their domain are used to get bounds on the degree of approximation better than those obtained by using just the smoothness (e.g., \cite{eignet, mhaskar2020kernel}) or those obtained by using just the bounds on the kernels (as done commonly in the literature on dimension independent bounds, see Section~\ref{sec:relatedwork}).

All of these results have the form (cf. Section~\ref{bhag:measures} for definitions)
\be\label{eq:genest}
\inf\limits_{P\in \mathcal{V}_N(G)}\|f-P\|_{\infty,\XX}\le CN^{-s}\|\tau\|_{TV},
\ee
for $f\in \mathcal{V}(G)$ subject to various conditions on $G$ and the measure $\tau$ defining $f$.
Here, $s>0$ and  $C>0$ are  constants independent of $f$ (and hence, $\tau$) but may depend in an unspecified manner on $G$, $\XX$, and $\YY$, and the conditions on $\tau$. 
In particular, when $\XX$ and $\YY$ are subsets of a Euclidean space, they may depend upon the dimension of these spaces. 
The bound \eqref{eq:genest} is called \emph{dimension independent} if $s$ is greater than some positive number independent of the dimension, and \emph{tractable} if in addition, $C$ depends at most polynomially on the dimension.

The purpose of this paper is to examine the conditions which allow us to obtain estimates of the form \eqref{eq:genest} that are both dimension independent and tractable, where the dimensions  of the spaces $\XX$ and $\YY$ are defined in an appropriate manner.
The outline of the paper is as follows.
In Section~\ref{sec:relatedwork}, we review some related works. 
The set up including notation and most of the assumptions is described in Section~\ref{sec:setup}.
The main theorems are stated in an abstract setting in Section~\ref{sec:mainresults}, and illustrated with examples in Section~\ref{sec:examples} related to ReLU$^\gamma$ networks, zonal function networks, and certain radial basis function networks called the Laplace networks. 
After developing some preliminary theory in Section~\ref{sec:prelim}, the proof of all the results in Sections~\ref{sec:mainresults} and \ref{sec:examples} are given in Sections~\ref{sec:pf_main} and \ref{sec:example_pf} respectively.

\section{Related Works}\label{sec:relatedwork}

We note at the outset that we are not aware of any work other than \cite{mhaskar2020dimension} dealing with $G$-networks on general metric spaces. 
The theory is very well developed on Euclidean domains. It is not possible to give an exhaustive survey, but we point out a few papers just to illustrate the kind of results that are available in the literature.

Usually, dimension independent bounds are obtained using probabilistic arguments, and as such, are not constructive. An important exception is the so-called Korobov spaces (or hyperbolic cross spaces) \cite{temlyakov1986approximation,novak2008tractability,dick2010digital, hackbusch2012tensor} which are defined on tensor product domains (typically torus or cube) in terms of certain mixed derivatives of the target functions. For example, functions from a  Korobov space with mixed derivatives of second order can be written as an integral of a kernel:
\begin{equation*}
    f( x)=\int_{\TT^q}\prod\limits_{j=1}^q\mathcal{B}(x_j-t_j)\frac{\partial^{2q}f}{\partial x_1^2\dots\partial x_q^2}(\mathbf{t})d\mathbf{t},
\end{equation*}
where $\displaystyle\mathcal{B}(x)=\sum\limits_{k\neq0}\frac{e^{ikx}}{(ik)^2}$. So such a space is a variation space.
Montanelli and Du \cite{montanelli2019new} proved the optimal  approximation rate $\tilde O(N^{-2})$\footnote{In this section, $\tilde O$ means that powers of $\log N$ are ignored, while $O$ means that they are not present.} for Korobov spaces by using deep neural networks. After this, \cite{mao2022approximation} proved the same rate by using deep convolutional neural networks with tractable constants. 
Approximation theorems for mixed H\"older smoothness classes are obtained by D\~ung and Nguyen in \cite{dung2021deep}, for which the constant term is given explicitly.
More generally, Suzuki \cite{suzuki2018adaptivity} proved the optimal rate  $\tilde O(N^{-\beta})$ for Besov spaces of mixed smoothness $\beta>0$ with unspecified constants. It is worth noting the mixed smoothness spaces can be described \textbf{only on tensor product spaces}, and the required smoothness of the functions is proportional to the dimension. So the function spaces are not large even if the functions are defined on high-dimensional domains.

The problem of approximating functions from $\mathcal{V}(G)$ by elements in $\mathcal{V}_N(G)$ has been studied widely in the literature. In 1993, Barron \cite{Barron1993} proved for functions with $\int_{\RR^d}\left|\hat f(\omega)\right||\omega|d\omega<\infty$ the dimension independent rate $O(N^{-1/2})$ of shallow sigmoid neural networks with respect to the $L_2$-norm. A periodic version of this result was obtained in  \cite{dimindbd}, where dimension independent bounds for shallow periodic neural networks are obtained for the class of continuous periodic functions whose Fourier coefficients are summable. The results are unimprovable in the sense of widths. As to the problem of approximating functions in the variation space by linear combination of elements in the dictionary, DeVore and Temlyakov \cite{devore1996some} showed that the rate $O(N^{-1/2})$ also holds for Hilbert spaces generated by orthogonal dictionaries using the greedy algorithm.
 K\r{u}rkov$\acute{a}$ and Sanguineti \cite{kurkova1,kurkova2} proved the rate $O(N^{-1/2})$ for Hilbert spaces generated by dictionaries with  conditions weaker  than orthogonality. The corresponding constants in these works are tractable. 

Although $O(N^{-1/2})$ is the optimal rate for general dictionaries, this rate can be improved for particular cases. However, in most literature for the improved rates, the corresponding constant terms are not necessarily tractable.

For the $L_2$-approximation, Xu \cite{xu2020finite} considered  the approximation of spectral Barron spaces using ReLU$^\gamma$ neural networks, $\gamma\ge 1$ integer, where the corresponding constant is tractable. Also, the sharp rate $\tilde O(N^{-\frac{1}{2}-\frac{2\gamma+1}{2d}})$ for the variation space generated by ReLU$^\gamma$ network is proved by Siegel and Xu \cite{siegel2022sharp} without tractable constants. The improved uniform approximation rates are also studied. Klusowski and Barron \cite{Barron2018} proved the rate $\tilde O(N^{-\frac{1}{2}-\frac{\gamma}{d}})$ for approximating functions from spectral Barron spaces of order $\gamma=\{1,2\}$ by shallow ReLU$^\gamma$ neural networks. They also proved absolute constants in this work. Using the covering number argument as in  \cite{makovoz1998uniform}, Ma, Siegel, and Xu \cite{ma2022uniform} recently obtained the uniform approximation rate $\tilde O(N^{-\frac{1}{2}-\frac{\gamma-1}{d+1}})$ for approximating functions in spectral Barron spaces of order $\gamma$ by ReLU$^\gamma$ networks with unspecified constants. 
All these results are applicable in the case $\gamma\ge 1$ is an integer.

In \cite{tractable}, dimension independent bounds of the form $\tilde O(N^{-1/2})$ are obtained for general $G$-networks on non-tensor product Euclidean domains, including neural, radial basis function, and zonal function networks, where also the constants involved in the estimates are dependent polynomially on the dimension. 
The paper explains a duality between the tractability of quadrature and approximation from closed, convex, symmetric hulls of dictionaries. They depend only on the boundedness properties of the dictionaries rather than taking into account also any smoothness properties of the kernels. 
Naturally, our results in \cite{mhaskar2020dimension} are sharper when both the results there and the ones in \cite{tractable} are applicable, but the constant terms are unspecified.

Our paper seeks to prove the results in \cite{mhaskar2020dimension} with slightly modified conditions to ensure that the constants involved in the estimates are tractable.
In particular, our results are valid for approximation of functions on arbitrary metric spaces, and hold for nonsymmetric kernels $G$; i.e., when the approximation of a function is sought on one metric space given data on another metric space. 
An important example is when the data is on a metric space, and the function approximation is sought on a larger metric space.

\section{Set up}\label{sec:setup}

In this section, we describe our basic set up. 
In 
Section~\ref{bhag:metricspace}, 
we summarize the necessary concepts and notation related to the metric spaces, including our notion of the dimension of a metric space.
We are interested in the approximation of functions of the form $x\mapsto \int_\YY G(x,y)d\tau(y)$. 

Section~\ref{bhag:kernels} deals with the properties of the kernels $G$ in which we are interested in this paper, and ideas related to measure theory are reviewed in Section~\ref{bhag:measures}.

\subsection{Metric spaces}\label{bhag:metricspace}

\begin{definition}\label{def:ballsandspheres} {\rm[Balls and spheres]}
    Let $(X,\rho)$ be a metric space. A \textbf{ball} on $(X,\rho)$ is a set
    \begin{equation}\label{eqn:defball}
        \BB_{X,\rho}( x,\delta):=\{ y\in X:\ \rho( x, y)\leq\delta\},\quad  x\in X,\ \delta>0.
    \end{equation}
    A \textbf{sphere} is the boundary of a ball defined as
    \begin{equation}\label{eqn:defsphere}
        \partial\BB_{X,\rho}( x,\delta):=\{ y\in X:\ \rho( x, y)=\delta\},\quad  x\in X,\ \delta>0.
    \end{equation}
    When the distance $\rho$ is clear from the context, we omit $\rho$ to write $\BB_{X}( x,\delta)$ and $\partial\BB_{X}( x,\delta)$. Likewise, we will omit $X$ from the notation if we do not expect any confusion.
\end{definition}

If $A\subset X$, it is convenient to denote $\BB_X(A,\delta)=\bigcup\limits_{ x\in A}\BB_X( x,\delta)$. We denote the closure of $X\setminus\BB_X(A,\delta)$ by $\Delta(A,\delta)$.

Next, we define the notion of the dimension of a metric space in terms of covering numbers.
\begin{definition}\label{def:covering}
{\rm [$\epsilon$-covering number and $\epsilon$-net]}
    Given $\epsilon>0$ and a compact metric space $(X,\rho)$, the \textbf{$\epsilon$-covering number} for a compact subset $A\subset X$ is defined as
    $$N_{\rho}(A,\epsilon):=\min\left\{n\in\NN:\ \exists\  y_1,\dots, y_n\in X,\hbox{ s.t. }A\subset\bigcup\limits_{k=1}^n\BB_{X,\rho}( y_k,\epsilon)\right\},$$
    the set $\{ y_k\}_{k=1}^n$ is called an \textbf{$\epsilon$-net} of $A$.
    
    When the metric $\rho$ is clear from the context, we omit $\rho$ to write $N(A,\epsilon)$.
\end{definition}

\begin{definition}\label{def:setdimension}
{\rm [Dimension of a family of sets]}
    Let $d\geq0$, $( X,\rho)$ a metric space. a compact subset $A$ of $ X$ is called (at most) \textbf{$d$-dimensional} if
    \begin{equation}\label{eqn:dim_set}
        \mfc_{A,\rho}:=\max\left\{1,\sup\limits_{\epsilon\in(0,1]}N(A,\epsilon)\epsilon^d\right\}<\infty
    \end{equation}
    A family $\mathcal{F}$ is called (at most) \textbf{$d$-dimensional} if
    \begin{equation}\label{eqn:dim_class}
        \mfc_{\mathcal{F},\rho}:=\sup\limits_{A\in\mathcal{F}}\mfc_{A,\rho}<\infty.
    \end{equation}
    When the metric $\rho$ is clear in the context, we omit $\rho$ to write $\mfc_{A}$ and $\mfc_{\mathcal{F}}$.
\end{definition}

\begin{remark}\label{rem:cover_dist}
 Although the dimension of a metric space itself is invariant under scaling of the metric, Example~\ref{ex:sphere} below demonstrates that the constants $\mfc_{A,\rho}$ and $   \mfc_{\mathcal{F},\rho}$ are not. 
In this paper, it will be assumed tacitly that the diameter of the metric space $\YY$ satisfies the following normalization.
\be\label{eq:diameter}
\mathsf{diam}(\YY)=\sup_{y_1,y_2\in\YY}\rho(y_1,y_2)=2.
\ee 
\end{remark}
We elaborate this definition and the various constants involved in the context of a Euclidean sphere.
The concepts introduced in this example will also be applied in the proof of Theorems~\ref{thm:ReLUr1}, \ref{thm:ReLUr2}, and  \ref{thm:zonal}.

\begin{example}\label{ex:sphere}
{\rm
For $Q\in\NN$, we define the Euclidean sphere embedded in $\RR^{Q+1}$ by
\begin{equation}\label{eqn:defQsphere}
    \SS^Q:=\left\{ \x\in\RR^{Q+1}:\ \sum\limits_{j=1}^{Q+1}|x_j|^2=1\right\}.
\end{equation}
The purpose of this example is to illustrate the dependence on the constants involved in Definition~\ref{def:setdimension} on the definition of a metric on $\SS^Q$.

The following proposition, proved by Böröczky and Wintsche in \cite[(1) and Corollary 1.2]{boroczky2003covering} (using different notations), plays an important role in our estimations of the various constants.
\begin{proposition}\label{prop:coversphere}
Let $Q\in\NN$, $Q\ge 2$, $\SS^Q$ be the unit ball in $\RR^{Q+1}$ and $\rho^*$ be the geodesic distance on $\SS^Q$. Then there exists an absolute constant $\kappa_\SS\ge 1$ with the following properties: For  $0<\delta\leq\pi/2$, $\SS^Q$ can be covered by
\begin{equation}\label{eqn:coveringnumber_spheres}
    (\kappa_\SS/2)\cdot\cos\delta\frac{1}{\sin^Q\delta}Q^{3/2}\log(1+Q\cos^2\delta)\le \kappa_\SS \frac{Q^{3/2}\log Q}{\sin^Q\delta}\cos\delta
\end{equation}
spherical balls of radius $\delta$. Furthermore, each point $ \x\in\SS^Q$ can belong to at most
$\kappa_\SS Q\log Q$
of these balls.
\end{proposition}

Let $\rho^*$ be the geodesic distance on $\SS^Q$, then $\SS^Q$ can be covered by $\displaystyle \kappa_\SS Q^{3/2}\log Q\cos\epsilon\frac{1}{\sin^Q\epsilon}$ balls of radius $\epsilon$. Using the relation $\displaystyle\frac{\pi}{2}\sin\epsilon\geq\epsilon$, we get
\begin{equation}\label{eqn:covering_sphere_SQ}
    N_{\rho^*}(\SS^Q,\epsilon)\leq \kappa_\SS Q^{3/2}\log Q\cos\epsilon\frac{1}{\sin^Q\epsilon}\leq \kappa_\SS(\log Q)Q^{3/2}\left(\frac{\pi}{2}\right)^Q\epsilon^{-Q}.
\end{equation}
In this case, the dimension of $\SS^Q$ is $Q$ and $\mfc_{\SS^Q,\rho^*}\leq \kappa_\SS Q^{3/2}\log Q\left(\frac{\pi}{2}\right)^Q$.

If we define the $\rho_1$-distance by
\begin{equation}\label{eqn:rho1_modify}
    \rho_1( \x, \y)=\frac{2}{\pi}\rho^*( \x, \y),\quad \x, \y\in\SS^Q,
\end{equation}
then \eqref{eqn:covering_sphere_SQ} shows that
\begin{equation}\label{eqn:coveringnumber_spheres_rho1}
    N_{\rho_1}(\SS^Q,\epsilon)=N_{\rho^*}\left(\SS^Q,\frac{\pi}{2}\epsilon\right)\leq \kappa_\SS Q^{3/2}\log Q\epsilon^{-Q}, \qquad 0<\epsilon\le 1,
\end{equation}
and the term $\mfc_{\SS^Q,\rho_1}$ is tractable; i.e., depends at most polynomially on $Q$.

Further, if we define the $\rho_2$-distance by
\begin{equation}\label{eqn:rho2_modify}
    \rho_2( \x, \y)=\frac{1}{\pi \kappa_\SS}\rho^*( \x, \y),\quad \x, \y\in\SS^Q,
\end{equation}
then
$$
N_{\rho_2}(\SS^Q,\epsilon)= N_{\rho^*}\left(\SS^Q,\pi \kappa_\SS\epsilon\right)\leq(2\kappa_\SS)^{-Q}\kappa_\SS Q^{3/2}\log Q\frac{1}{\sin^Q\left(\frac{\pi}{2}\epsilon\right)}.
$$
Since $2\log 2\ge 3/2+ 1/\log 2$, the function $Q\mapsto 2^{-Q}Q^{3/2}\log Q$ is decreasing on $[2,\infty)$. In addition, since $\kappa_\SS^{Q-1}\ge 1$, we conclude that for $Q\ge 2$,
\begin{equation}\label{eqn:coveringnumber_spheres_rho2}
N_{\rho_2}(\SS^Q,\epsilon)\leq \epsilon^{-Q}, \qquad \epsilon< 1/(2\kappa_\SS),
\end{equation}
so that $\mfc_{\SS^Q,\rho_2}\leq1$.
\qed}\end{example}

\subsection{Kernels}\label{bhag:kernels}
In the rest of this paper, we will consider two compact metric spaces $(\XX,\rho_{\XX})$ and $(\YY,\rho_\YY)$ and write $G:\XX\times\YY\to\RR$ as a kernel on $\XX\times\YY$.
A motivating example is the  ReLU function $( \x\cdot \y)_+$, $\x\in\SS^Q$, $\y\in\SS^q$ for some positive integer $q\le Q$. 
As a function of $ \x$, this is Lipschitz continuous on $\SS^Q$.
As a function of $ \y$, we take a closer look and observe that it is infinitely differentiable away from the equator $\{ \y\in\SS^q :  \x\cdot \y=0\}$, while on the equator, it is Lipschitz continuous. 

In order to define these notions of smoothness in the abstract, we 
let $\{\Pi_k\}$  be a nested sequence of finite dimensional subspaces of $C(\YY)$: $\Pi_1\subset\Pi_2\subset\dots$ with the dimension of $\Pi_k$ being $D_k$. It is convenient to extend the notation to non-integer values of $k$ by setting $\Pi_k=\Pi_{\lfloor k\rfloor}$ and $D_k=D_{\lfloor k\rfloor}$.
For any $A\subset X$,  $C(A)$ denotes the class of bounded, real-valued, uniformly continuous functions on $A$,
equipped with the supremum norm $\|\cdot\|_A$.
For any $A\subset X$, $f\in C(A)$, $r>0$, let
    $$E_r(A;f)=\inf\limits_{P\in\Pi_r}\|f-P\|_A.$$
\begin{definition}\label{def:locsmoothness}{\rm [Local smoothness]}
    Let $r>0$, $Y\subset \XX$, and $f\in C(Y)$.
    The function $f$ is called \textbf{$r$-smooth on  $Y$} if
    \begin{equation}\label{eqn:deflocalsmooth}
        \|f\|_{Y,r}=\sup_{\delta>0}\sup\limits_{ x\in Y}\frac{E_{r}(\BB_Y( x,\delta);f)}{\delta^r}<\infty.
    \end{equation}
\end{definition}

Let $T\subset\XX$. Given $ x\in\XX$, the smoothness of the function $G( x,\circ)$ could be large outside a low-dimensional subset $\Ex\subset T$. Hence we can make the following assumptions.

\begin{definition}\label{def:kerneldef}
Let $T\subset\YY$. A kernel $G:\XX\times\YY\to\RR$ is called a \textbf{kernel of class} $\mathcal{G}(\alpha,r,R,u,T)$ if each of the following conditions are satisfied
\begin{itemize}
    \item (\textbf{H\"older continuity}): The kernel $G(\circ, y)$ is H\"older $\alpha$ continuous on $\XX$:
    \begin{equation}\label{eqn:Lipschitz}
        |G|_{\XX,\alpha}:=\sup\limits_{ y\in T}\sup\limits_{ x\neq z}\frac{|G( x, y)-G(z, y)|}{\rho_{\XX}( x,z)^\alpha}<\infty.
    \end{equation}

    \item (\textbf{Global smoothness}): $G( x,\circ)$ is $r$-smooth on $\YY$ with
    \begin{equation}\label{eqn:G_smoothness_global}
        |G|_r:=\sup\limits_{ x\in\XX}\|G( x,\circ)\|_{ \YY,r}<\infty.
    \end{equation}
    \item (\textbf{Smoothness in the large}) For every $ x\in\XX$, there exists a compact set $\Ex=\Ex(G)\subset T$ with the following property. For every $\delta>0$, $G( x,\circ)$ is $R$-smooth on $\Delta(\Ex,\delta)$ with
    \begin{equation}\label{eqn:G_smoothness_local}
        |G|_{\Delta,R,u}:=\sup\limits_{\delta>0}\sup\limits_{ x\in\XX}\delta^u\|G( x,\circ)\|_{\Delta(\Ex,\delta),R}<\infty.
    \end{equation}
    In this case, we have
    $$\|G( x,\circ)\|_{\Delta(\Ex,\delta),R}\leq|G|_{\Delta,R,u}\delta^{-u},\qquad x\in\XX,\ \delta>0.$$

\end{itemize}
We define a seminorm on $\mathcal{G}(\alpha,r,R,u,T)$ by
\begin{equation}\label{eqn:def_G_seminorm}
    |G|_{\mathcal{G}}:=\max\left\{|G|_{\XX,\alpha},|G|_r,|G|_{\Delta,R,u}\right\}.
\end{equation}

\end{definition}

The following examples illustrate the definition in the case of two of the important kernels we are interested in.

\begin{example}\label{ex:relu}
{\rm
    Let $Q,q\in\NN$ and $Q\geq q$. We consider the case when $\XX=\SS^Q$,
    $$\YY=\left\{ \x\in\RR^{Q+1}:\ \sum\limits_{j=1}^{q+1}x_j^2=1,\ x_j=0\hbox{ for }j>q+1\right\}\subset\SS^Q,$$
    and $G( \x, \y)=( \x\cdot \y)_+^\gamma$ for some $\gamma>0$ and $\rho_\YY=\rho^*$ be the geodesic distance on $\SS^Q$. For each $k\geq1,$ let $\Pi_k$ be the set of spherical polynomials of degree $<k$. Then $G( \x,\circ)$ is $\gamma$-smooth on $\SS^Q$ for each $ \x\in\SS^Q$. We take $\Exx :=\SS^Q$ for $ \x\in\{\bz\in\SS^Q:\ z_1=\dots=z_{q+1}=0\}$ and take $\Exx := \{ \y\in\SS^Q: \x\cdot \y = 0\}$ otherwise. If $\gamma$ is an integer,  $G( \x,\circ)$ is a spherical polynomial of degree $\gamma$ on the set $\{\y\in\SS^Q:\ \x\cdot\by>0\}$ and equal to $0$ on $\{\y\in\SS^Q:\ \x\cdot\y<0\}$. Therefore, for any set $A\subset\SS^Q\setminus\Exx$, $\|G( \x,\circ)\|_{A,R}=0$ for any $R\in\NN$.  If $\gamma$ is not an integer, then for any such set $A$ and $R>\gamma$, $G( \x, \circ)$ is $R$-times differentiable, but $\|G( \x,\circ)\|_{A,R}\leq 2^\gamma\kappa_\SS^{\gamma-R}\dist_{\rho^*}(\Exx,A)^{\gamma-R}$ (see Section \ref{subsec:q=Q,gammanotN}). \qed
    }
\end{example}

\begin{example}\label{uda:laplacekern}{\rm
    We consider the case $\XX=\YY=B^Q$, $G( \x, \y)=\exp(-| \x- \y|)$. It is clear that for each $ \x\in B^Q$, $G( \x,\circ)$ is 1-smooth on $B^Q$.  For any $A\subset\YY\setminus\{ \x\}$, $G( \x,\circ)$ is infinitely differentiable, and $\|G( \x,\circ)\|_{A,R}\leq R!\dist( \x,A)^{-R}$. So we take $\Exx=\{ \x\}$, $r=\alpha=1$.\qed}
\end{example}

\subsection{Measures}\label{bhag:measures}

In the sequel, the term measure will refer to a signed or complex Borel measure (or positive measure having a bounded total variation) on a metric space $\YY$. 
The total variation measure $|\tau|$ of a signed measure $\tau$ on $\YY$ is defined by 
$$
|\tau|(A)=\sup\sum_{j}|\tau(U_j)|,
$$
where the sum is over all countable partitions of $A$ into Borel measurable sets $U_j\subseteq\YY$.
We will denote $|\tau|(\YY)=|\tau|_{TV}$.
The support $\supp(\tau)$ is the set of all $y\in\YY$ for which $|\tau|(\BB(y,\delta))>0$ for every $\delta>0$. It is easy to see that $\supp(\tau)$ is a compact subset of $\YY$.

A measure $\tau$ is said to be non-atomic if for any measurable $A\subset \YY$ with $\tau(A)>0$, there exists a measurable subset $B\subset A$ with $0<\tau(B)<\tau(A)$. 

In this paper, we will require the measure $\tau$ involved in the definition of the target functions to have certain properties, which are summarized in the following definition.

\begin{definition}\label{def:admissible}
    Let $(\YY,\rho_{\YY})$ be a compact metric space. A measure $\tau$ is called an \textbf{admissible measure} on $(\YY,\rho_{\YY})$ if
    \begin{itemize}
        \item $\tau$ is non-atomic;
        \item $\tau$ has a finite total variance $|\tau|_{TV}$;
        \item The $\tau$-measure of the spheres of $\YY$ is zero:
        \begin{equation}\label{eq:spheremeasure}
    \tau(\partial\BB_\YY(y,\epsilon))=0,\quad y\in\YY,\ \epsilon>0.
        \end{equation}
    \end{itemize}
\end{definition}

\section{Main Results}\label{sec:mainresults}
Given compact metric spaces $(\XX,\rho_{\XX})$ and $(\YY,\rho_{\YY})$, and a kernel $G: \XX\times\YY\to \RR$ satisfying the conditions in Definition~\ref{def:kerneldef}, we are interested in approximating certain functions in $\mathcal{V}(G)$, defined with measures $\tau$ satisfying the conditions in Definition~\ref{def:admissible}.

Theorem~\ref{thm:main} is a general theorem governing the approximation of such functions.

\begin{theorem}\label{thm:main}
Let $\alpha,\ u,\ R,\ r>0$, $Q,\ q,\ s\in\NN$ satisfy
\begin{equation}\label{eqn:cond_thm1}
    Q\geq q\geq s,\quad R\geq r,\quad q-s+2u>2R-2r.
\end{equation}
Let $(\XX,\rho_{\XX})$ be a $Q$-dimensional compact metric spaces, $\tau$ be an admissible measure on a compact metric space $(\YY,\rho_{\YY})$ with a $q$-dimensional support $T$.
Let $G:\XX\times\YY\to\RR$ be a kernel in $\mathcal{G}(\alpha,r,R,u,T)$. 

Let $f:\XX\to\RR$ be defined by
\be\label{eq:fdef}
f( x)=\int_ T G( x, y)d\tau( y),
\ee
where $\tau$ is an admissible measure.

In addition, we assume that 
\begin{equation}\label{eqn:def_Theta}
    |\tau|(\BB_ T(\Ex(G),\epsilon))\leq\Theta_{\tau,G}|\tau|_{TV}\epsilon^{q-s},\qquad x\in\XX,\ \epsilon>0.
\end{equation}
for some $\Theta_{\tau,G}<\infty$.

Let $\mfc_T$ as in \eqref{eqn:dim_set} and
\begin{equation}\label{eqn:def_a_thm1}
    a=\frac{2R-2r}{q-s+2u}\in[0,1).
\end{equation}
Then for
\begin{equation}\label{eqn:cond_main1}
    N\geq 3\mfc_T(D_R+2)(q-s)^{\frac{q}{1-a}}
\end{equation}
there exists $\{ y_1,\dots, y_N\}\subset\YY$ and numbers $a_1,\dots,a_N$ with $\sum\limits_{k=1}^N|a_k|\leq|\tau|_{TV}$ such that
\begin{equation}\label{eqn:main1}
    \left\|f-\sum\limits_{k=1}^Na_kG(\circ, y_k)\right\|_\XX\leq c_1|G|_{\mathcal{G}}|\tau|_{TV}\frac{\sqrt{1+\log N}}{N^{\frac{1}{2}+\frac{R-ua}{q}}},
\end{equation}
where
\begin{equation}\label{eqn:thm_main_constant}
 c_1=4e[\mfc_T(3D_R+6)]^{\frac{1}{2}+\frac{R-ua}{q}}\left[\left(Q\frac{q+2R-2ua}{2\alpha}+\log\mfc_\XX\right)^{1/2}\mfc_T^{-1/2}\left(\Theta_{\tau,G}+1\right)^{1/2}+1\right].   
\end{equation}

\end{theorem}

The following example illustrates the role of $\Theta_{\tau,G}$ in \eqref{eqn:def_Theta}.
\begin{example}\label{ex:musphere}
{\rm
    We consider the case when $\XX=\YY=\SS^Q$, $\rho^*$ is the geodesic distance on $\SS^Q$, and $G( \x, \by)=( \x\cdot \y)_+^\gamma$ for some $\gamma>0$. Again, we take $\Exx:= \{ \by\in\SS^Q: \x\cdot \by = 0\}$, and let $\tau=\mu^*$ be the volume measure normalized so that $\mu^*(\SS^Q)=1$. 
    We denote the volume measure of $\SS^Q$ by $\nu_Q$. Then it is verified easily that the $\mu^*$-measure of $\BB_{\SS^Q}(\Exx,\epsilon)$ satisfies
    \begin{equation*}
        \mu^*(\BB_{\SS^Q,\rho^*}(\Exx,\epsilon))\leq\frac{\nu_{Q-1}}{\nu_Q}\int_{\frac{\pi}{2}-\epsilon}^{\frac{\pi}{2}+\epsilon}\sin^{Q-1}\theta d\theta\leq2\sqrt{\frac{Q+2}{\pi}}\epsilon.
    \end{equation*}
    Thus we can take $s=Q-1$ and $\displaystyle\Theta_{\tau,G}=\Theta_{\mu^*,G}=2\sqrt{\frac{Q+2}{\pi}}$.
\qed}\end{example}

Under certain conditions, Theorem \ref{thm:main} can be improved. Specifically, we have the following result.
\begin{theorem}\label{thm:main2}
Let $\alpha,\ u,\ R,\ r>0$, $Q,\ q\in\NN$ satisfy
$$Q\geq q,\quad R\geq r,$$
Let $(\XX,\rho_{\XX})$ be a $Q$-dimensional compact metric spaces, $\tau$ be an admissible measure on a compact metric space $(\YY,\rho_{\YY})$ with a $q$-dimensional support $T$.
Let $G:\XX\times\YY\to\RR$ be a kernel in $\mathcal{G}(\alpha,r,R,0,T)$ satisfying
\begin{itemize}
    \item[(a)]$\Ex$ is either empty set or a set of only one point,
    \item[(b)]$|G|_{\Delta,R,0}<\infty$.
\end{itemize}
Then for $f$ as defined in \eqref{eq:fdef}, $\mfc_T$ as in \eqref{eqn:dim_set}, and
\begin{equation}\label{eqn:cond_main2}
    N\geq 3\mfc_T(D_R+2),
\end{equation}
there exists $\{ y_1,\dots, y_N\}\subset \YY$ and numbers $a_1,\dots,a_N$ with $\sum\limits_{k=1}^N|a_k|\leq|\tau|_{TV}$ such that
\begin{equation}\label{eqn:main2}
    \left\|f-\sum\limits_{k=1}^Na_kG(\circ, y_k)\right\|_\XX\leq c_1'|G|_{\mathcal{G}}|\tau|_{TV}\frac{\sqrt{1+\log N}}{N^{\min\left(\frac{1}{2}+\frac{R}{q},1+\frac{r}{q}\right)}}.
\end{equation}
where
\begin{equation}\label{eqn:thm_main2_constant}
 c_1'=8\mfc_T^{1+\frac{R}{q}}(3D_R+6)^{1+\frac{R}{q}}\left[\left(Q\frac{q+2R}{2\alpha}+\log\mfc_\XX\right)^{1/2}\mfc_T^{-1/2}+1\right].   
\end{equation}

\end{theorem}

In the next section, we will apply Theorem \ref{thm:main} and Theorem \ref{thm:main2} to various examples. As we have discussed in Remark \ref{rem:cover_dist}, we will take $\rho_\YY$ such that $\mathrm{diam}(\YY)=2$, for which $\mfc_T$ become tractable constants in these examples.

\section{Examples}\label{sec:examples}
The purpose of this section is to illustrate the general Theorems~\ref{thm:main} and \ref{thm:main2}. In Section~\ref{bhag:relu}, we consider power ReLU functions.
Positive definite zonal function networks as in \cite{zfquadpap} are discussed in Section~\ref{bhag:zonal}.
Approximation on the unit ball by Laplace networks is discussed in Section~\ref{bhag:laplace}.

\subsection{Approximation by ReLU networks}\label{bhag:relu}

In Examples \ref{ex:sphere}, \ref{ex:relu}, \ref{ex:musphere} and Proposition \ref{prop:coversphere}, we have studied some properties of the spheres and ReLU$^\gamma$ functions. Our results in Section \ref{sec:mainresults} can be applied in these settings to get the approximation rates and tractable constants.
\begin{theorem}\label{thm:ReLUr1}
Let $\gamma\geq1$, $Q\geq q\geq2$ be integers, and $\kappa_\SS$ be the absolute constant in \eqref{eqn:coveringnumber_spheres}. We can identify the unit ball of dimension $q$ as
\begin{equation}\label{eqn:Sq_in_SQ}
\SS^q:=\left\{\x\in\RR^{Q+1}:\ \sum\limits_{j=1}^{q+1}|x_j|^2=1\hbox{ and } x_j=0 \hbox{ for } j> q+1\right\}\subset\SS^Q
\end{equation}
where $\SS^Q$ is the unit ball of dimension $Q$. Let  $\rho^*$ be the geodesic distance on $\SS^q$, $\tau$ be an admissible measure with $\supp(\tau)=\SS^q$ on $(\SS^q,\rho^*)$. We assume that there exists $\Xi_\tau<\infty$ such that
\begin{equation}\label{eqn:cond_Ex_sphere}
    |\tau|(\BB_{\SS^q,\rho^*}(\y,\delta))\leq\Xi_\tau|\tau|_{TV}\mu^*(\BB_{\SS^q,\rho^*}(\y,\delta)),\quad \y\in\SS^q,\ \delta>0,
\end{equation}
where $\mu^*$ is the volume measure normalized so that $\mu^*(\SS^q)=1$.

Let $G:\SS^Q\times\SS^q\to\RR$ be the $\mathrm{ReLU}^\gamma$ function
$$G(\x,\by)=(\x\cdot\by)_+^\gamma,$$
and $f$ is a function on $\SS^q$ denoted by
$$f(\x)=\int_{\SS^q}G(\x,\by)d\tau(\by),\quad\x\in\SS^q.$$
Then for any $R\geq \gamma$ and $N\geq\kappa_\SS q^{3/2}(3(q+1)^{\gamma+1}+6)\log q$, there exists $\{\by_1,\dots,\by_N\}\subset\SS^q$ and numbers $a_1,\dots,a_N$, such that
\begin{equation}\label{eqn:rateReLUr1}
    \left\|f-\sum\limits_{k=1}^Na_kG(\circ,\by_k)\right\|_{\SS^Q}\leq c_2|\tau|_{TV}\left\{\begin{array}{ll}\displaystyle\frac{\sqrt{1+\log N}}{N^{\frac{1}{2}+\frac{\gamma}{q}+\frac{\lambda}{2q}}},\quad&\hbox{if }q=Q,\\[3ex]
\displaystyle\frac{\sqrt{1+\log N}}{N^{\frac{1}{2}+\frac{\gamma}{q}}},\quad&\hbox{if }q<Q,\\
\end{array}\right.
\end{equation}
where $\displaystyle\lambda=\frac{2R-2\gamma}{2R-2\gamma+1}$ and
\begin{equation}
c_2=
\left\{
\begin{aligned}
\displaystyle &16\sqrt{\pi}e\kappa_\SS (2\pi)^R\left[\kappa_\SS q^{3/2}(3(q+1)^{R+1}+6)\log q\right]^{\frac{1}{2}+\frac{R}{q}}\left[(q+2R+\log(\kappa_\SS q^2))(6\Xi_\tau)^{1/2}+1\right],&\\
&&\mbox{if }q=Q,\\[2ex]
\displaystyle &16\kappa_\SS \pi^\gamma\left[\kappa_\SS q^{3/2}(3(q+1)^\gamma+6)\log q\right]^{1+\frac{\gamma}{q}}\left[\left(\frac{Q(q+2\gamma)}{2}+\log(\kappa_\SS Q^2)\right)^{1/2}\!\!(\kappa_\SS q^{3/2}\log q)^{-1/2}+1\right],& \\
&&\mbox{if }q<Q.\\
\end{aligned}\right.
\end{equation}
\end{theorem}
\begin{remark}\label{rem:Xi_Theta}
    {\rm We note that the condition \eqref{eqn:cond_Ex_sphere} leads to a bound of the term $\Theta_{\tau,G}$ in \eqref{eqn:def_Theta} (cf. \eqref{eqn:thetaGest_thm4.1}), after changing the metric from geodesic distance to a multiple of this distance. We feel that \eqref{eqn:cond_Ex_sphere} is more natural and understandable than a condition on $\BB_{\SS^q}(\Ex,\epsilon)$.\qed}
\end{remark}

\begin{remark}\label{rem:dimdifference}
    {\rm The difference in the estimates in the cases $Q=q$ and $Q>q$ is caused by $\Exx$. For $Q=q$, we can take $\Exx$ be the equator $\Exx=\{\by\in\SS^q:\ \x\cdot\by=0\}$ and the function $G$ is arbitrarily smooth outside $\Exx$. 
    Naturally, the measure of the tube $\BB_{\SS^q}(\Exx,\epsilon)$, under proper conditions, can be estimated as $O(\epsilon^{q-1})$. 
    So, we can apply Theorem \ref{thm:main} with $R>r$ and $s=1$.
    
    However, if $Q>q$, let $\x=(0,\dots,0,1)\in\SS^Q$, this set is given as $\{\by\in\SS^q:\ 0y_1+\dots+0y_q=0\}=\SS^q$. This means we have to take $\Exx=\varnothing$ and take $R=r$ globally.\qed}
\end{remark}
We can improve upon \eqref{eqn:rateReLUr1} in the case when $\gamma$ is an integer and $q=Q$.
\begin{theorem}\label{thm:ReLUr2}
Under the conditions in Theorem \ref{thm:ReLUr1}, if $\gamma$ is an integer and $Q=q$, we have
\begin{equation}\label{eqn:rateReLUr2}
     \left\|f-\sum\limits_{k=1}^Na_kG(\circ,\by_k)\right\|_{\SS^Q}\leq c_2'|\tau|_{TV}\displaystyle\frac{\sqrt{1+\log N}}{N^{\frac{1}{2}+\frac{2\gamma+1}{2q}}},
\end{equation}
where
\begin{equation}\label{eqn:integer_gamma_const}
c_2'=32\sqrt{\pi}e\kappa_\SS(2\pi)^\gamma\left[\kappa_\SS q^{3/2}(3(q+1)^\gamma+6)\log q\right]^{1+\frac{\gamma+1}{q}}\left[(q+\gamma+1+\log(\kappa_\SS q^2))(6\Xi_\tau)^{1/2}+1\right].
\end{equation}

\end{theorem}

\begin{remark}\label{rem:reluabs}{\rm  Motivated by a direct comparison with the constructive results in \cite{sphrelu}, we have considered in  \cite{mhaskar2020dimension} kernels of the form $G(\x,\by)=|\x\cdot\by|^\gamma$, $\x\in\SS^Q$, $\y\in\SS^q$. It is clear that  when $\gamma$ is an integer, $|\x\cdot\by|^\gamma=(\x\cdot\by)_+^\gamma+(-\x\cdot\by)_+^\gamma$. So the approximation rate we get here is the same as the rate in \cite{mhaskar2020dimension}. This means we are not losing the approximation rate in exchange for the tractability. \qed}
\end{remark}

\subsection{Approximation by certain zonal function networks}\label{bhag:zonal}
\begin{theorem}\label{thm:zonal}
Let $Q\geq q\geq2$ be integers, $\gamma>0$ be not an integer, and  $\kappa_\SS $ be the absolute constant in \eqref{eqn:coveringnumber_spheres}. Let $\SS^Q$ the unit ball in $\RR^{Q+1}$, and we identify $\SS^q$ as \eqref{eqn:Sq_in_SQ}. Let  $\rho^*$ be the geodesic distance on $\SS^q$ and $\tau$ be an admissible measure on $(\SS^q,\rho^*)$ with $\mathsf{supp}(\tau)=\SS^q$.

Let $G:\SS^Q\times\SS^q\to\RR$ be the zonal function
$$G(\x,\by)=(1-\x\cdot\by)^\gamma,$$
and $f$ is a function on $\SS^Q$ denoted by
$$f(\x)=\int_{\SS^q}G(\x,\by)d\tau(\by),\quad\x\in\SS^Q.$$
Then for any $N\geq\kappa_\SS q^{3/2}\log q(3(q+1)^\gamma+6)$, there exists $\{\by_1,\dots,\by_N\}\subset\SS^q$ and numbers $a_1,\dots,a_N$, such that
\begin{equation}\label{eqn:ratezonal}
    \left\|f-\sum\limits_{k=1}^Na_kG(\circ,\by_k)\right\|_{\SS^Q}\leq c_3|\tau|_{TV}\left(\frac{1+\log N}{N^{1+4\gamma/q}}\right)^{1/2},
\end{equation}
where
\begin{equation}\label{eqn:thm_zonal_const}
c_3=8\kappa_\SS \pi^{2\gamma}\left[\kappa_\SS q^{3/2}(3(q+1)^{2\gamma+1}+6)\log q\right]^{1+\frac{2\gamma}{q}}\left[\left(\frac{Q(q+2\gamma)}{2}+\log(\kappa_\SS Q^2)\right)^{1/2}\left(\kappa_\SS q^{3/2}\log q\right)^{-\frac{1}{2}}+1\right].
\end{equation}
\end{theorem}

\subsection{ Approximation on the unit ball by radial basis function networks}\label{bhag:laplace}
\begin{theorem}\label{thm:Gaussian}
Let $q,Q\in\NN$, $q\leq Q$, denote the unit balls
\begin{equation*}
    \begin{split}
        B^q:=\left\{\x\in\RR^Q:\ \sum\limits_{j=1}^q|x_j|^2\leq1,\ x_{q+1}=\dots=x_Q=0\right\},\quad B^Q:=\left\{\x\in\RR^Q:\ \sum\limits_{j=1}^Q|x_j|^2\leq1\right\}.
    \end{split}
\end{equation*}
Let $\tau$ be an admissible measure on
$(B^q,|\cdot|)$ with $\mathsf{supp}(\tau)=B^q$. Let $G:B^Q\times B^q\to\RR$ be the Laplace function
$$G(\x,\by)=\exp\left(-|\x-\by|\right):=\exp\left(-\sqrt{\sum\limits_{j=1}^q|x_j-y_j|^2+\sum\limits_{j=q+1}^Q|x_j|^2}\right),$$
and $f$ is a function on $B^Q$ denoted by
$$f(\x)=\int_{B^q}G(\x,\by)d\tau(\by),\quad\x\in B^Q.$$
Then for any $N\geq 3q+9$, there exists $\{\by_1,\dots,\by_N\}\subset\SS^q$ and numbers $a_1,\dots,a_N$, such that
\begin{equation}\label{eqn:ratelaplace}
    \left\|f-\sum\limits_{k=1}^Na_kG(\circ,\by_k)\right\|_{B^Q}\leq c_4|\tau|_{TV}\left(\frac{1+\log N}{N^{1+2/q}}\right)^{1/2},
\end{equation}
where
\begin{equation}\label{eqn:thm_Gaussian_const}
 c_4=8\left[(\kappa_Bq^{3/2}\log q)(3q+9)\right]^{1+\frac{1}{q}}\left[\left(\frac{q+2}{2}Q+2\log Q+\log\kappa_B\right)^{1/2}\left(\kappa_Bq^{3/2}\log q\right)^{-1/2}+1\right]   
\end{equation}
and $\kappa_B$ is an absolute constant.

\end{theorem}

\section{Preliminaries for the Proofs}\label{sec:prelim}
Using an obvious scaling and the Hahn decomposition theorem, it suffices to prove Theorems~\ref{thm:main} and  \ref{thm:main2} assuming that $|G|_{\mathcal{G}}=1$ (cf. \eqref{eqn:def_G_seminorm}) and that $\tau$ is a probability measure. Thus, we assume throughout this section and Section \ref{sec:pf_main} that
\begin{equation}\label{eqn:tau_TV=1}
    |G|_{\mathcal{G}}=1,\quad|\tau|_{TV}=1,\quad\tau(Y)\geq0,\quad Y\subset\YY.
\end{equation}

The basic idea behind our proof is the same as that of the proof of the main theorem, Theorem~3.1 in \cite{mhaskar2020dimension}. 
Thus, we first construct a partition of the support of $\tau$ that enables us to take advantage of the smoothness properties of $G$ described in Definition~\ref{def:kerneldef}. 
The main technical novelty of our paper is the construction of the partition which needs to be done more carefully than in \cite{mhaskar2020dimension} to ensure the tractability of the constants involved. This is described in Section~\ref{bhag:partition}.
On each subset in this partition, we will consider the set of positive quadrature measures exact for integrating elements of $\Pi_R$.
Then we use the ideas in \cite{bourgain1988distribution} to define a probability measure on the set of such measures.
This part is described in Section~\ref{bhag:quadrature}.
The proof is completed using H\"offding's inequality and its consequences, described in Section~\ref{bhag:concentration}.

\subsection{``Partition" on $T$ and $\varepsilon$-net on $\XX$}\label{bhag:partition}

In Lemma \ref{lem:partition} below, we provide our ``partition" of $T=\mathsf{supp}(\tau)$; i.e., a finite collection of closed subsets $\{A_k\}$ such that $T\subset\bigcup\limits_{k=1}^M A_k$ and $|\tau|(A_k\cap A_j)=0$ if $k\not=j$. In all the examples in Section~\ref{sec:examples}, we will verify in Section~\ref{sec:example_pf} that all of the constants in $\tau(A_k)$, $E_R(A_k;G( x,\circ))$, and $E_r(A_k;G( x,\circ))$, etc. are tractable, which ensures the desired tractability property.
In this section, we assume the metric space to be $(\YY, \rho_\YY)$, and its mention will be ommitted from the notation.

\begin{lemma}\label{lem:partition}
Suppose $\tau$ is a probability measure on $\YY$.
We assume that  $T=\supp(\tau)$ satisfies
$$N\left( T,\epsilon\right)\leq\mfc_T\epsilon^{-q},\qquad \epsilon>0$$
and
\begin{equation}\label{eqn:tau_sphere}
    \tau(\partial\BB( x,\epsilon))=0,\quad x\in T,\ \epsilon>0.
\end{equation}
Then for any $\epsilon>0$, there exists closed subsets $\{A_k\}_{k=1}^{M}$ and points $ y_1,\dots, y_M \in \YY$ such that $T\subset\bigcup\limits_{k=1}^MA_k$,
\begin{equation}\label{eqn:number_of_triangles}
    M\leq3\mfc_T\epsilon^{-q},
\end{equation}
\begin{equation}\label{eqn:partition_measure_new}
    \tau(A_k)\leq\mfc_T^{-1}\epsilon^{q};
\end{equation}
\begin{equation}\label{eqn:partition_diam_new}
     A_k\subset\BB\left( y_k, \epsilon\right),\quad k=1,\dots,M,
\end{equation}
and each $A_k\cap A_j$ lies in a finite union of spheres in $T$ of type \eqref{eqn:defsphere}. As a result,
\begin{equation}\label{eqn:A_intersection0}
    \tau\left(A_k\cap A_j\right)=0,\quad j\neq k.
\end{equation}

\end{lemma}

The idea behind the proof is the following. 
An obvious partition obtained from the covering of $T$ by balls of radius $\epsilon$  divides the set $ T$ into $\left\lfloor\mfc_T\epsilon^{-q}\right\rfloor$ subsets of radius $\leq\epsilon$. We divide those subsets which have measures greater than $\mfc_T^{-1}\epsilon^q$  into smaller subsets with measures between $\left(2 \mfc_T\epsilon^{-q}\right)^{-1}$ and $\left(\mfc_T\epsilon^{-q}\right)^{-1}$, and show that the number of such subsets is bounded by $2\left\lfloor \mfc_T\epsilon^{-q}\right\rfloor$ in total.

The following lemma gives the details of this subdivision.
\begin{lemma}\label{lem:divideY}
    Let $Y\subset T$ and the boundary of $\overline Y$ lie in a finite union of spheres (cf. \eqref{eqn:defsphere}). Then for any $ \xi>0$, there exists a partition $\{U_\ell\}_{\ell=1}^{L}$ with     $\displaystyle L\leq1+2\tau(Y) \xi^{-1}$
    such that
    $$\tau(U_\ell)\leq \xi,\quad\ell=1,\dots,L.$$
\end{lemma}
\begin{proof}
Our first step is to divide $Y$ into subsets with $\tau$-measure less than $\frac{1}{2} \xi$.

For each $ x_0\in T$, by the non-atomic property and the continuity from above, we have
$$\lim\limits_{n\to\infty}\tau(\BB( x_0,n^{-1}))=\tau\left(\bigcap\limits_{n=1}^{\infty}\BB( x_0,n^{-1})\right)=\tau(\{x_0\})=0.$$
Then there exists $\delta_{ x_0}>0$ such that $\tau(\BB(x_0,\delta_{x_0}))<\frac{1}{2} \xi$. If $x_0\in\YY\setminus T$, then there exists $\delta_{x_0}$ such that $\tau(\BB(x_0,\delta_{x_0}))=0$. Since $\YY$ is compact, there exists some $\delta>0$ such that $\tau(\BB(x,\delta))<\frac{1}{2} \xi$ for all $x\in\YY$.

There exist points $ y_1,\dots, y_J\in \YY$ with $J:=\left\lfloor \mfc_T\delta^{-q}\right\rfloor\geq N(Y,\delta)$ and the relative balls in $Y$
$$\BB_ Y( y_j,\delta)=\{ y\in Y:\ \rho_{\YY}( y, y_j)<\delta\}= \BB( y_j, \delta)\cap Y,\quad j=1,\dots,J$$
such that $ Y=\bigcup\limits_{j=1}^J\BB_ Y( y_j,\delta)$. Let
\begin{equation}\label{eqn:constrY_k}
    B_1=\BB_ Y( y_1,\delta),\quad B_j=\BB( y_j,\delta)\setminus\left(\bigcup\limits_{i=1}^{j-1}\BB_ Y( y_i,\delta)\right),\quad j=2,\dots,J.
\end{equation}
Then $\{B_j\}_{j=1}^J$ is a partition of $ Y$ satisfying $B_j\subset\BB_ Y( y_j,\delta)$ and $\tau(B_j)\leq\frac{1}{2} \xi$ for each $j$.

Also, it is clear from the construction that the boundary of $\overline{B_j}$'s lie in a finite union of spheres with form \eqref{eqn:defsphere}. We will construct our desired subsets of $Y$ from this partition.

Let $n_0=0$, we define the integers $n_\ell$'s, $\ell=1,2,\cdots$, by
\begin{equation}\label{eqn:def_n_ell}
    n_\ell=\min\left\{n\in\NN:\ n>n_{\ell-1},\ \sum\limits_{j=n_{\ell-1}+1}^{n}\tau(B_{j})\geq\frac{1}{2} \xi\right\}
\end{equation}
for $\ell$ such that $n_{\ell-1}$ exists. Clearly, this procedure will stop in one of the two cases. One is $\sum\limits_{j=n_{L-1}+1}^{J}\tau(B_{j})<\frac{1}{2} \xi$, the other is $\sum\limits_{j=n_{L-1}+1}^{J}\tau(B_{j})\geq\frac{1}{2} \xi$ and $\sum\limits_{j=n_{L-1}+1}^{J-1}\tau(B_{j})<\frac{1}{2} \xi$ for some $L\in\NN$. In either case, we will denote $J$ by $n_{L}$. Then for $\ell\geq L+1$, the integer $n_{\ell}$ in \eqref{eqn:def_n_ell} does not exist. Take the unions of $B_{j}$'s by defining
\begin{equation}\label{eqn:constrU}
    U_{\ell}:=\bigcup\limits_{j=n_{\ell-1}+1}^{n_\ell}B_{j},\quad \ell=1,\dots,L.
\end{equation}
By definition, for $\ell=1,\dots,L-1$, we have
\begin{equation}\label{eqn:tauU_elllower}
    \tau(U_{\ell})=\sum\limits_{j=n_{\ell-1}+1}^{n_\ell}\tau(B_{j})\geq\frac{1}{2} \xi.
\end{equation}

Since each $n_\ell$ is taken as the minimum integer in \eqref{eqn:def_n_ell}, we have $\sum\limits_{j=n_{\ell-1}+1}^{n_\ell-1}\tau(B_{j})<\frac{1}{2} \xi$. Consequently,
\begin{equation}\label{eqn:tauU_ellupper}
    \tau(U_{\ell})=\tau(B_{n_\ell})+\sum\limits_{j=n_{\ell-1}+1}^{n_\ell-1}\tau(B_{j})<\frac{1}{2} \xi+\frac{1}{2} \xi= \xi.
\end{equation}
The estimation \eqref{eqn:tauU_elllower} implies
$$\tau(Y)\geq\sum\limits_{\ell=1}^{L-1}\tau(U_{\ell})\geq(L-1)\times\frac{1}{2} \xi,$$
we conclude our desired bound of $L$:
$$L\leq1+2\tau(Y) \xi^{-1}.$$
By our construction of $B_j$'s, the boundary of each $\overline{B_j}$ lies in the union of the boundary of $\overline Y$ and a sphere with form \eqref{eqn:defsphere}. By our assumption on $Y$, we conclude the boundary of each $\overline{B_j}$ lies in a finite union of spheres with form \eqref{eqn:defsphere}. Thus the boundary of each $\overline U_\ell$ also has this property.
\end{proof}

\begin{proof}[\textit{Proof of Lemma \ref{lem:partition}}]
We repeat the procedure \eqref{eqn:constrY_k} with $T$ in place of $ Y$ and $\epsilon$ in place of $\delta$. Then we get a partition $\{Y_{k}\}_{k=1}^{K}$ of $T$ with $K\leq\mfc_T\epsilon^{-q}$ and each $Y_{k}$ lies in a ball of radius $\epsilon$.
Then $\{Y_k\}_{k=1}^K$ is a partition of $ T$ satisfying $Y_k\subset\BB( y_k,\epsilon)$ for each $k$. To complete the proof, we apply Lemma \ref{lem:divideY} with $ \xi=\mfc_T\epsilon^{q}$ and each $Y_k$. We have the partitions $Y_k=\bigcup\limits_{\ell=1}^{L_k}U_{\ell,k}$, $k=1,\dots,K$. So we have
$$\sum\limits_{k=1}^{K}L_k\leq\sum\limits_{k=1}^{K}(1+2\tau(Y_k)\mfc_T\epsilon^{-q})\leq K+2\mfc_T\epsilon^{-q}\sum\limits_{k=1}^{K}\tau(Y_k)\leq K+2\mfc_T\epsilon^{-q}\leq3\mfc_T\epsilon^{-q}.$$
Rewrite the partition $\{U_{\ell,k}\}_{\ell=1,k=1}^{\ell=L_k,k=K}$ as $\{\tilde A_k\}_{k=1}^{M}$, then $M=\sum\limits_{k=1}^{K}L_k\leq3\mfc_T\epsilon^{-q}$ and \eqref{eqn:tauU_ellupper} implies $\tau(\tilde A_k)\leq \xi=\mfc_T\epsilon^{q}$. Finally, let $A_k:=\overline{\tilde A_k}$ for each $k$. By Lemma \ref{lem:divideY}, the boundaries of $\overline{\tilde A_k}$'s are necessarily contained in finite unions of spheres with form \eqref{eqn:tau_sphere}. Hence we have $\tau\left(A_k\cap A_j\right)=0$ and $\tau(A_k)=\tau(\tilde A_k)\leq \xi=\mfc_T\epsilon^{q}$. 
\end{proof}

\subsection{Quadrature and probability measures}\label{bhag:quadrature}
Next, we recall from \cite[Theorem~5.2]{mhaskar2020dimension} the construction of quadrature measures and the probability measure on the set of these measures. We will use the  quadrture formulas for the sets $A_k$ constructed as in Lemma~\ref{lem:partition}, exact for integrating elements of $\Pi_R$ on each of these sets. 
The existence of such measures is guaranteed, e.g., by Tchakaloff's theorem (cf. \cite[Theorem~5.3]{mhaskar2020dimension}). 

\begin{theorem}\label{thm:bourgaintheo}
Let $\YY$ be a compact topological space, $\{\psi_j\}_{j=0}^{M-1}$ be continuous real valued functions on $\YY$, and $\nu$ be a probability measure on $\YY$.
Let $\mathbb{P}_M(\YY)$ denote the set of all probability measures $\omega$ supported on at most $M+2$ points of $\YY$ with the property that
\be\label{abs_quadrature}
\int_\YY \psi_j( y)d\omega( y)=\int_{\YY}\psi_j( y)d\nu( y), \qquad j=0,\cdots, M-1.
\ee
Then $\nu$ is in the weak-star closed convex hull of $\mathbb{P}_M(\YY)$, and hence, there exists a measure $\omega^*_\YY$ on $\mathbb{P}_M(\YY)$ with the property that for any $f\in C(\YY)$,
\be\label{barycenter}
\int_\YY f( y)d\nu( y)=\int_{\mathbb{P}_M(\YY)} \left(\int_\YY f( y)d\omega( y)\right)d\omega^*_\YY(\omega).
\ee
\end{theorem}

\subsection{Concentration inequality}\label{bhag:concentration}
In this section, we recall the H\"offding inequality and prove a lemma that enables us to estimate errors in the uniform norms rather than pointwise errors.

The H\"offding's inequality \cite[Appendix~B, Corollary~3]{pollard2012bk} is given in the following Lemma~\ref{lem:hoeffdinglemma}.
\begin{lemma}\label{lem:hoeffdinglemma}
Let $X_1,\cdots, X_n$ be independent random variables with zero means and bounded ranges: $a_j\le X_j\le b_j$, $j=1,\cdots,n$. 
Then 
\be\label{hoeffdingineq}
\mathsf{Prob}\left(\left|\sum_{j=1}^n X_j\right| \ge t\right)\le 2\exp\left(-\frac{2t^2}{\sum_{j=1}^n (b_j-a_j)^2}\right), \qquad t>0.
\ee
\end{lemma}

\begin{lemma}\label{lem:global_estimation}
    Let $\varepsilon>0$, $\mathcal{C}$ be an $\varepsilon$-net for a metric space $(X,\rho_X)$. Let $\alpha>0$, $(W,\PP,\mathcal{B})$ be a probability space, $g(\circ; x)$ be a random process on $(W,\PP,\mathcal{B})$ with $ x\in\mathcal{C}$ and $|g|_{W,X,\alpha}:=\sup\limits_{w\in W}|g(w,\circ)|_{X,\alpha}<\infty$. (cf. \eqref{eqn:Lipschitz}).  Suppose there exists a positive number $\Lambda$ such that for each $ x\in\mathcal{C}$,
    \begin{equation}\label{eqn:cond_lemma6.4}
        \mathrm{Prob}\left(\left|g(w; x)-f( x)\right|\geq t\right)\leq2\exp\left(-\frac{t^2}{\Lambda}\right),\quad t>0,
    \end{equation}
    where $f( x)=\EE(g(\circ; x))$.
    Then there exist a choice of $w\in W$ such that $g(w; x)$ satisfies
    \begin{equation}
        \|f-g(w;\circ)\|_X\leq \sqrt{\Lambda\log\left(4|\mathcal{C}|\right)}+2|g|_{W,X,\alpha}\varepsilon^\alpha.
    \end{equation}
\end{lemma}

\begin{proof}
From \eqref{eqn:cond_lemma6.4}, we have
\begin{equation}
\begin{split}
    \mathrm{Prob}\left(\max\limits_{ x\in\mathcal{C}}|g(w; x)-f( x)|\geq t\right)\leq2|\mathcal{C}|\exp\left(-\frac{t^2}{\Lambda}\right).
\end{split}
\end{equation}
Choosing
\begin{equation}\label{eqn:choice_t0}
    t_0=\sqrt{\Lambda\log\left(4|\mathcal{C}|\right)},
\end{equation}
we obtain
$$\mathrm{Prob}\left(\max\limits_{ x\in\mathcal{C}}|g(w; x)-f( x)|\geq t_0\right)\leq\frac{1}{2}.$$
Then there exists a choice of $w$'s such that
$$\max\limits_{ x\in\mathcal{C}}|g(w; x)-f( x)|\leq t_0.$$

Now for every $ x'\in X$, there exists $ x\in\mathcal{C}$ with $\rho_{X}( x, x')\leq\varepsilon$. The condition \eqref{eqn:Lipschitz} then leads to the fact that
\begin{equation*}
    \begin{split}
        &\biggl||g(w; x')-f( x')|-|g(w; x)-f( x)|\biggr|\leq|g(w; x')-f( x')-g(w; x)+f( x)|\\
        \leq&|g|_{W,X,\alpha}\varepsilon^\alpha+|f|_{X,\alpha}\varepsilon^\alpha\leq2|g|_{W,X,\alpha}\varepsilon^\alpha.
    \end{split}
\end{equation*}

and
\begin{equation*}
    \begin{split}
        \|f-g(w;\circ)\|_X\leq&\max\limits_{ x\in\mathcal{C}}|g(w; x)-f( x)|+2|g|_{W,X,\alpha}\varepsilon^\alpha\leq t_0+2|g|_{W,X,\alpha}\varepsilon^\alpha\\
        \leq& \sqrt{\Lambda\log\left(4|\mathcal{C}|\right)}+2|g|_{W,X,\alpha}\varepsilon^\alpha.
    \end{split}
\end{equation*}
\end{proof}

We will construct in the next subsection an $\varepsilon$-net $\mathcal{C}\subset\XX$ for a proper $\varepsilon$ and construct a class of random variables $\Omega_k(\omega_k; x)$ satisfying
$$\sum\limits_{k=1}^M\Omega_k(\omega_k; x)=\GG(\{\omega_k\}; x)-f( x),\qquad x\in\mathcal{C},$$
where $\GG(\{\omega_k\};\circ)$ is a linear combination of the desired form in \eqref{eqn:main1}. We will get a realization $\GG(\{\omega_k\}; x)$ of the sum which approximates $f$ on $\mathcal{C}$ by using Hoeffding's inequality.

\section{Proofs of the Main Results}\label{sec:pf_main}

In this section, We prove Theorem \ref{thm:main} and Theorem \ref{thm:main2}. 
The proofs of both these theorems share a great deal of details in common: the construction of certain random variables and the estimation on the sum of the squares of their ranges. 
This part is presented in Section~\ref{bhag:randomvar}.
The proofs of Theorem~\ref{thm:main} and \ref{thm:main2} are presented in Sections~\ref{bhag:pfthm1} and \ref{bhag:pfthm2} respectively.

\subsection{The random variables}\label{bhag:randomvar}

To construct our desired random variables, we fix $\epsilon>0$ to be chosen later,  use the ``partition" $\{A_k\}_{k=1}^M$ as in Lemma \ref{lem:partition}.

We apply Theorem \ref{thm:bourgaintheo} to each $A_k$ with a basis $\{\psi_j\}_{j=1}^{D_R}$ of $\Pi_R$, and $\tau_k=\frac{1}{\tau(A_k)}\tau$ in place of $\nu$. This gives a measure $\omega_k^*$ on $\PP_{D_R}(A_k)$ such that
\begin{equation}\label{eqn:partition_quadrature}
    \int_{A_k} fd\tau_k=\frac{1}{\tau(A_k)}\int_{A_k}fd\tau=\int_{\PP_{D_R}(A_k)}\left(\int_{A_k}fd\omega\right)d\omega_k^*,\quad f\in C(A_k),
\end{equation}
In particular,
\begin{equation}
    \int_{A_k} Pd\tau=\int_{\PP_{D_R}(A_k)}\left(\tau(A_k)\int_{A_k}Pd\omega\right)d\omega_k^*,\quad P\in\Pi_R.
\end{equation}
We define a family of independent random variables $\{\Omega_k\}$ on $\PP_{D_R}$, with each $\Omega_k$ having $\omega_k^*$ as the probability law, by
$$\Omega_k(\omega)=\tau(A_k)\int_{A_k}G( x, y)d\omega( y)-\int_{A_k}G( x, y)d\tau( y),\quad\omega\in\PP_{D_R}(A_k).$$
For any realization of these random variables, $\omega_k\in\PP_{D_R}(A_k)$, we write
\begin{equation}\label{eqn:defGG}
    \GG(\{\omega_k\}; x)=\GG( x)=\sum\limits_{k=1}^M\tau(A_k)\int_{A_k}G( x, y)d\omega_k( y).
\end{equation}

The random variables $\Omega_k$ satisfy
\begin{equation}\label{eqn:Omega_sum}
    \begin{split}
        \sum\limits_{k=1}^M\Omega_k(\omega_k)=&\sum\limits_{k=1}^M\tau(A_k)\int_{A_k}G( x, y)d\omega( y)-\sum\limits_{k=1}^M\int_{A_k}G( x, y)d\tau( y)\\
        =&\GG( x)-\int_ T G( x, y)d\tau( y)=\GG( x)-f( x)
    \end{split}
\end{equation}
since $\tau(A_k\cap A_j)=0$ for all $k\neq j$.

In view of Lemma \ref{lem:global_estimation}, we  only need to estimate the term $\prob\left(\left|\GG( x)-f( x)\right|\geq t\right)$ pointwise for each $x\in\XX$, and choose $\varepsilon$ judiciously to get the estimates in Theorems~\ref{thm:main} and \ref{thm:main2}.

Given \eqref{eqn:Omega_sum}, we will estimate the probability that $\left|\sum\limits_{k=1}^M\Omega_k(\omega_k)\right|\geq t$ for $t>0$ and choose this $t$ later. To this end, we need to consider the sum of the squares of the ranges of the random variables $\Omega_k$ to apply Hoeffding’s inequality \eqref{hoeffdingineq}. This estimation  is the key step in our proof.

In view of the definition of $\PP_{D_R}(A_k)$, we see that for every $\omega_k\in\PP_{D_R}(A_k)$ and $P\in\Pi_R$,
\begin{equation}\label{eqn:polynomial_appr}
    \tau(A_k)\int_{A_k}G( x, y)d\omega_k(y)-\int_{A_k}G( x, y)d\tau( y)=\tau(A_k)\int_{A_k}(G( x, y)-P( y))d\omega_k( y)-\int_{A_k}(G( x, y)-P( y))d\tau( y).
\end{equation}
Let $ x\in\XX$. Since $A_k\subset\BB\left( y_k,\epsilon\right)$ for each $k$, we deduce from \eqref{eqn:deflocalsmooth}
that for every $\omega_k\in\PP_{D_R}(A_k)$ and $\delta\ge \epsilon$,
\begin{equation}\label{eqn:bounding_Omega_k_1}
    \begin{split}
|\Omega_k(\omega_k)|\leq&2\tau(A_k)\inf\limits_{P\in\Pi_R(A_k)}\|G( x,\circ)-P\|_{A_k}\\[1ex]
        \leq&2\tau(A_k)\min\left\{E_R(A_k;G( x,\circ)),E_r(A_k;G( x,\circ))\right\}\\[1ex]
        \leq&2\tau(A_k)\left\{\begin{array}{ll}
        \epsilon^r\|G( x,\circ)\|_{A_k,r},\quad& \hbox{if }A_k\cap\BB(\Ex,\delta)\not=\varnothing,\\[2ex]
        \epsilon^R\|G( x,\circ)\|_{A_k,R,u},\quad&\mbox{otherwise.}
        \end{array}\right.
    \end{split}
\end{equation}
We have assumed $|G|_{\mathcal{G}}=1$ in \eqref{eqn:tau_TV=1}. Hence, if $A_k\cap\BB(\Ex,\delta)\not=\varnothing$, we have $\|G( x,\circ)\|_{A_k,r}\leq\|G( x,\circ)\|_{\XX,r}\leq 1$ and likewise (cf. \eqref{eqn:G_smoothness_local}), $\|G( x,\circ)\|_{A_k,R}\leq\|G( x,\circ)\|_{\Delta(\Ex,\delta),R}\leq \delta^{-u}$   when $A_k\cap\BB(\Ex,\delta)=\varnothing$. Then we get from \eqref{eqn:G_smoothness_global} and \eqref{eqn:G_smoothness_local} that
\begin{equation}\label{eqn:bounding_Omega_k_2}
    |\Omega_k(\omega_k)|\leq\beta_k=2\tau(A_k)\left\{\begin{array}{ll}
      \epsilon^r, \quad & \hbox{if }A_k\cap\BB(\Ex,\delta)\not=\varnothing, \\[1ex]
       \epsilon^R\delta^{-u}, \quad & \hbox{otherwise}.
    \end{array}\right.
\end{equation}
In view of \eqref{eqn:partition_measure_new}, we have $\tau(A_k)\le \mfc_T^{-1}\epsilon^{q}$ for each $k$. We deduce that
\begin{equation}\label{eqn:sumbeta1}
    \begin{split}
        &\sum\limits_{k=1}^M\beta_k^2=\sum\limits_{k:A_k\cap\BB_ T(\Ex,\delta)\neq\varnothing}\beta_k^2+\sum\limits_{k:A_k\cap\BB_ T(\Ex,\delta)=\varnothing}\beta_k^2\\
\leq&\sum\limits_{k:A_k\cap\BB_T(\Ex,\delta)\neq\varnothing}\tau(A_k)\left(4\tau(A_k)\epsilon^{2r}\right)+\sum\limits_{k:A_k\cap\BB_T(\Ex,\delta)=\varnothing}\tau(A_k)\left(4\tau(A_k)\epsilon^{2R}\delta^{-2u}\right)\\
 \leq&4\mfc_T^{-1}\epsilon^q\epsilon^{2r}\left(\sum\limits_{k:A_k\cap\BB_T(\Ex,\delta)\neq\varnothing}\tau(A_k)\right)+4\mfc_T^{-1}\epsilon^q\epsilon^{2R}\delta^{-2u}\left(\sum\limits_{k:A_k\cap\BB_T(\Ex,\delta)=\varnothing}\tau(A_k)\right)\\
        \leq& 4\mfc_T^{-1}\epsilon^q\left\{\epsilon^{2r}\sum\limits_{k:A_k\cap\BB_T(\Ex,\delta)\neq\varnothing}\tau(A_k)+\epsilon^{2R}\delta^{-2u}\sum\limits_{k:A_k\cap\BB_T(\Ex,\delta)=\varnothing}\tau(A_k)\right\}.
            \end{split}
\end{equation}

In the proof of Theorems~\ref{thm:main} and \ref{thm:main2}, we will use slightly different arguments to arrive at a bound on $\sum \beta_k^2$ in terms of $\epsilon$ and then use Lemma~\ref{lem:global_estimation} with juidicious choices of $\Lambda$ and $\varepsilon$ to arrive at an estimation of $\|\mathbb{G}(\{\omega_k\},\circ)-f\|_\XX$ in terms of $\epsilon$.
To complete the proofs in both cases, we now describe the choice of $\epsilon$ in terms of a given budget $N$ of neurons.

We note that  $\GG(\{\omega_k\};\circ)$ is a function with the form $\displaystyle\sum\limits_{k=1}^Na_kG(\circ, y_k)$ in \eqref{eqn:main1} and \eqref{eqn:main2}. To see this, we recall each $\omega_k$ is a discrete probability measure supported in at most $D_R+2$ points, so we can write
$$\tau(A_k)\int_{A_k}G( x, y)d\omega_k( y)=\tau(A_k)\sum\limits_{j=1}^{D_R+2}b_{k,j}G(\circ,\tilde y_{k,j}).$$
Together with the fact $\sum\limits_{k=1}^{M}\tau(A_k)=1$, we see the function $\GG(\{\omega_k\};\circ)$ has the form $\displaystyle\sum\limits_{k=1}^Na_kG(\circ, y_k)$ with (cf. \eqref{eqn:number_of_triangles})
\begin{equation}\label{eqn:N_coeff_bds}
  N\le M(D_R+2)\le 3\mfc_T(D_R+2)\epsilon^{-q}, \qquad \sum_{k=1}^N |a_k|\le |\tau|_{TV}.  
\end{equation}
Thus, given $N\ge 1$, we choose
\begin{equation}\label{eqn:epsilon_choice}
   \epsilon=\left(\frac{3\mfc_T(D_R+2)}{N}\right)^{1/q}. 
\end{equation}
In view of \eqref{eqn:epsilon_choice}, and the fact that $3\mfc_T(D_R+2)\ge 1$,
we deduce that
\begin{equation}\label{eqn:logepsilon_bd}
 \log(1/\epsilon)\le \log N.   
\end{equation}

\subsection{Proof of Theorem \ref{thm:main}}\label{bhag:pfthm1}
If $A_k\cap\BB(\Ex,\delta)\neq\varnothing$, we have $A_k\subset\BB(\Ex,\delta+2\epsilon)$. 
Hence, (cf. \eqref{eqn:def_Theta})
$$
\sum\limits_{k:A_k\cap\BB(\Ex,\delta)\neq\varnothing}\tau(A_k)\le \tau(\BB(\Ex,\delta+2\epsilon))\leq \Theta_{\tau, G}(\delta+2\epsilon)^{q-s}.
$$
Further,
$$
\sum\limits_{k:A_k\cap\BB(\Ex,\delta)=\varnothing}\tau(A_k)\le \tau(\YY)=1.
$$
Therefore,
\eqref{eqn:sumbeta1} becomes
\begin{equation}\label{eqn:thm1_sumbeta1}
        \sum\limits_{k=1}^M\beta_k^2\leq 4\mfc_T^{-1}\epsilon^q\left\{\Theta_{\tau,G}(\delta+2\epsilon)^{q-s}\epsilon^{2r} +\delta^{-2u}\epsilon^{2R}\right\}
\end{equation}
We can balance the two terms $(\delta+2\epsilon)^{q-s}\epsilon^{2r}$ and $\delta^{-2u}\epsilon^{2R}$ by taking $\delta=\epsilon^{a}$ with $a$ as in \eqref{eqn:def_a_thm1}. Clearly, $\displaystyle\frac{\epsilon}{\delta}=\epsilon^{1-a}<1$. If $\displaystyle\frac{\epsilon}{\delta}\leq\frac{1}{q-s}$, i.e., 
\begin{equation}\label{eqn:cond_on_epsilon}
    \epsilon\leq (q-s)^{-\frac{1}{1-a}},
\end{equation}
we have
\be
(\delta+2\epsilon)^{q-s}\leq\delta^{q-s}\left(1+\frac{2}{q-s}\right)^{q-s}\leq e^2\delta^{q-s}=e^2\epsilon^{a(q-s)}.
\ee

Substituting the choice of $\delta$ into \eqref{eqn:thm1_sumbeta1}, we obtain
\begin{equation}\label{eqn:bound_sum_beta}
    \sum\limits_{k=1}^M\beta_k^2\leq 4\mfc_T^{-1}\left(e^2\Theta_{\tau,G}+1\right)\epsilon^{q+2R-2ua}
\end{equation}
This estimation and the H\"offding inequality \eqref{hoeffdingineq} gives us the pointwise upper bound
\begin{equation}\label{eqn:prob_G-f_C}
    \prob\left(\left|\GG(\{\omega_k\}, x)-f( x)\right|\geq t\right)\leq2\exp\left(-\frac{t^2}{8\mfc_T^{-1}\left(e^2\Theta_{\tau,G}+1\right)\epsilon^{q+2R-2ua}}\right),\quad t>0.
\end{equation}

We now take 
$$\varepsilon=\epsilon^{\frac{q+2R-ua}{2\alpha}},$$
so that there exists an $\varepsilon$-net $\mathcal{C}$ of $\XX$ that satisfies $|\mathcal{C}|\leq\mfc_\XX\varepsilon^{-Q}$(cf. \eqref{eqn:dim_set}). Then applying Lemma \ref{lem:global_estimation} with 
$$\Lambda=8\mfc_T^{-1}\left(e^2\Theta_{\tau,G}+1\right)\epsilon^{q+2R-2ua},$$
we deduce the existence of some $\{\tilde\omega_k\}$ such that
\begin{equation}\label{eqn:f-G_with_epsilon}
    \begin{split}
        &\|f-\GG(\{\tilde\omega_k\},\circ)\|_\XX\leq2\Lambda^{1/2}\left(Q\log\left(\frac{1}{\varepsilon}\right)+\log\mfc_\XX\right)^{1/2}+2\varepsilon^\alpha\\
        \leq&8\mfc_T^{-1/2}\left(Q\frac{q+2R-2ua}{2\alpha}\log\frac{1}{\epsilon}+\log\mfc_\XX\right)^{1/2}\left(e^2\Theta_{\tau,G}+1\right)^{1/2}\epsilon^{q/2+R-ua}+2\epsilon^{q/2+R-ua}\\
        \leq& 4e\mfc_T^{-1/2}\left(\Theta_{\tau,G}+1\right)^{1/2}\left[\max\left(Q\frac{q+2R-2ua}{2\alpha},\log\mfc_\XX\right)\left(\log\frac{1}{\epsilon}+1\right)\right]^{1/2}\left(2\epsilon^{q/2+R-ua}\right)+\left(2\epsilon^{q/2+R-ua}\right)\\
        \leq& 8e\left[\mfc_T^{-1/2}\left(\Theta_{\tau,G}+1\right)^{1/2}\left(Q\frac{q+2R-2ua}{2\alpha}+\log\mfc_\XX\right)^{1/2}+1\right]\left(\log\frac{1}{\epsilon}+1\right)^{1/2}\epsilon^{q/2+R-ua}.
        \end{split}
        \end{equation}

\vskip 1ex

In light of equations \eqref{eqn:epsilon_choice} and \eqref{eqn:logepsilon_bd},  this completes the proof of \eqref{eqn:main1}.
 We note that our choice of $\epsilon$ in \eqref{eqn:epsilon_choice} shows that \eqref{eqn:cond_on_epsilon} is equivalent to the condition \eqref{eqn:cond_main1} on $N$.\qed


\subsection{Proof of Theorem \ref{thm:main2}}\label{bhag:pfthm2}
In this theorem, we suppose $\Ex$ is a set of at most one point for each $ x\in\XX$, and $u=0$ for some $R\geq r$.

Without loss of generality, we assume $\Ex$ is contained in the interior of some $A_k$. Then since $u=0$, we can simply take $\delta$ small enough so that $A_j\cap\BB( x,\delta)=\varnothing$ for all $j\neq k$. In this case, we can estimate \eqref{eqn:sumbeta1} without more conditions on $\epsilon$ similar to \eqref{eqn:cond_on_epsilon} as follows, where we recall that $\mfc_T\ge 1$.
\begin{equation}\label{eqn:sumbeta_epsilon_main2}
        \sum\limits_{k=1}^M\beta_k^2\leq 4\mfc_T^{-1}\epsilon^{2r+q}\max\limits_{1\leq k\leq M}\tau(A_k)+4\mfc_T^{-1}\epsilon^{2R+q}
        \leq 4\mfc_T^{-2}\epsilon^{2r+2q}+4\mfc_T^{-1}\epsilon^{2R+q}
\leq 4\mfc_T^{-1}(\epsilon^{2r+2q}+\epsilon^{2R+q}).
\end{equation}
This estimation and \eqref{hoeffdingineq} gives us the pointwise upper bound
\begin{equation}\label{eqn:prob_G-f_C_main2}
    \prob\left(\max\limits_{ x\in\mathcal{C}}\left|\GG(\{\omega_k\}, x)-f( x)\right|\geq t\right)\leq2\exp\left(-\frac{t^2}{8\mfc_T^{-1}(\epsilon^{2r+2q}+\epsilon^{2R+q})}\right),\quad t>0.
\end{equation}
We now use Lemma~\ref{lem:global_estimation} with $$\varepsilon=\epsilon^{\frac{2R+q}{2\alpha}}, \qquad
\Lambda=8\mfc_T^{-1}(\epsilon^{2r+2q}+\epsilon^{2R+q})$$
to deduce the existence of $\{\tilde\omega_k\}$ such that, as in \eqref{eqn:f-G_with_epsilon},
\begin{equation}\label{eqn:f-G_with_epsilon_main2}
\begin{split}
     &\left\|f-\GG(\{\tilde\omega_k\};\circ)\right\|_\XX\\
     \leq&2\Lambda^{1/2}\left(Q\log\left(\frac{1}{\varepsilon}\right)+\log\mfc_\XX\right)^{1/2}+2\varepsilon^\alpha\\
     \leq&2\sqrt{2}\left[\left(Q\frac{q+2R}{2\alpha}+\log\mfc_\XX\right)^{1/2}\left(8\mfc_T^{-1}\right)^{1/2}+1\right]\sqrt{1+\log \left(\frac{1}{\epsilon}\right)}\max\left(\epsilon^{q/2+R},\epsilon^{q+r}\right)\\
     \leq&8\left[\left(Q\frac{q+2R}{2\alpha}+\log\mfc_\XX\right)^{1/2}\mfc_T^{-1/2}+1\right]\sqrt{1+\log \left(\frac{1}{\epsilon}\right)}\max\left(\epsilon^{q/2+R},\epsilon^{q+r}\right).
\end{split}
\end{equation}
In light of equations \eqref{eqn:epsilon_choice} and \eqref{eqn:logepsilon_bd},  this completes the proof of \eqref{eqn:main2}.
 \qed

\section{Proof of theorems in Section~\ref{sec:examples}. }\label{sec:example_pf}
In this section, we prove the theorems in Section~\ref{sec:examples}.
The proofs consist of applying the main theorems in Section~\ref{sec:mainresults}, for which we need to explain what the spaces $\Pi_k$ are in each case, and estimate the various parameters involved, such as $D_k$, $\Theta_{\tau, G}$, $u$, $\mfc_T$, etc.
We prove Theorems~\ref{thm:ReLUr1}, \ref{thm:ReLUr2}, \ref{thm:zonal}, and \ref{thm:Gaussian} in Sections~\ref{sec:pfcor1}, \ref{sec:pfcor1_2}, \ref{sec:pfcor2}, and \ref{sec:pfcor3} respectively.

\subsection{Proof of Theorem~\ref{thm:ReLUr1}}\label{sec:pfcor1}
\subsubsection{Proof for $q=Q$}\label{subsec:q=Q,gammanotN}
 We are going to estimate all the quantities in Theorem \ref{thm:main} for our particular case here.
\newline
\emph{Step 1. $\Pi_k$ and upper bound of $D_k$}
\newline
For any $k\geq1$, let $\Pi_k$ be the space of polynomials on $\SS^q$ of degree $<k$. Then
\begin{equation}\label{eqn:dimDk_sphere}
D_k\leq\binom{q+\lfloor k\rfloor}{\lfloor k\rfloor}\leq(q+1)^k,\quad k\geq1.
\end{equation}
\newline
\emph{Step 2. Upper bound of $\mfc_\XX$ and $\mfc_T$}
\newline
We have studied the set $\SS^q$ in our previous examples in Section \ref{sec:setup}. In order to prove the tractability, it is natural to consider $\SS^q$ as a metric space with diameter equal to $2$ (cf. Remark \ref{rem:cover_dist}). To this end, we define
\begin{eqnarray*}
&&\rho_{\XX}(\x,\by)=\rho_{\YY}(\x,\by)=\frac{2}{\pi}\rho^*(\x,\by),\quad\x,\by\in\SS^q,
\end{eqnarray*}
where $\rho^*$ be the geodesic distance on $ \XX=\YY=T=\SS^q$.
Then by Proposition \ref{prop:coversphere},
$$N_{\rho_{\YY}}(\SS^q,\epsilon)=N_{\rho_{\XX}}(\XX,\epsilon)\leq \kappa_\SS q^{3/2}\log q\cos\epsilon\frac{1}{\sin^q\epsilon}\left(\frac{2}{\pi}\right)^q\leq \kappa_\SS q^{3/2}\log q\epsilon^{-q}.$$

So we can use $\mfc_\XX=\kappa_\SS q^2$ and $\mfc_T= \kappa_\SS q^{3/2}\log q$.   

In this proof only, we write
$$\BB_{\SS^q}(\by,\epsilon):=\BB_{\SS^q,{\rho_{\YY}}}(\by,\epsilon).$$
    \newline
\emph{Step 3. Choice of $r$ and upper bound of $|G|_r$}
\newline
    Fix $\x\in\SS^q$, write $G_\x=G(\x,\circ)$ and take
    $$r:=\gamma.$$
    In each interval $I$ of length $2\delta_0$, the univariable function $x_+^\gamma$ can be approximated by the $\lfloor \gamma\rfloor$-degree Taylor polynomial $P_I$ at the midpoint of $I$ with
    $$\sup\limits_{x\in I}|x_+^\gamma-P_I(x)|\leq2^\gamma\delta_0^\gamma.$$
    For any ball $\BB_{\SS^q}(\by,\delta)\subset\SS^q$, we have
$$|\x\cdot\by-\x\cdot\by'|\leq|\by-\by'|\leq\max\limits_{\by'\in\BB_{\SS^q}(\by,\delta)}\rho^*(\by,\by')=\frac{\pi}{2}\delta,$$
where $|\cdot|$ is the Euclidean norm.
Then $\displaystyle\x\cdot\by'\in I_{\by,\delta}:=\left[\x\cdot\by-\frac{\pi}{2}\delta,\x\cdot\by+\frac{\pi}{2}\delta\right]$ and
    \begin{equation*}
        \begin{split}
            E_r(\BB_{\SS^q}(\by,\delta);G_\x)\leq\max\limits_{\by'\in\BB_{\SS^q}(\by,\delta)}|(\x\cdot\by')_+^\gamma-P_{I_{\by,\delta}}(\x\cdot\by')|\leq2^\gamma\left(\frac{\pi}{2}\delta\right)^\gamma=(\pi\delta)^\gamma.
        \end{split}
    \end{equation*}
So we have $|G|_r\leq\pi^\gamma$.
\newline
\emph{Step 4. Choice of $u$ and upper bound of $|G|_{\Delta,R,u}$}
\newline
    $\Exx$ is the equator of $\SS^q$ given by
    $$\Exx=\left\{\by_1\in\SS^q:\ \by_1\cdot\x=0\right\}.$$
    If an interval $I$ of length $2\delta_0$ satisfies $\hbox{dist}(I,0)\geq\delta_1$, then by using the Taylor polynomials again, $x_+^\gamma$ can be approximated by the polynomial $P_I$ of degree $\lfloor R\rfloor$ with
    $$\sup\limits_{x\in I}|x_+^\gamma-P_I(x)|\leq2^R\delta_1^{\gamma-R}\delta_0^R.$$
    For any $\delta,\ \tilde\delta>0$ and $A\subset\Delta(\Exx,\tilde\delta)\cap\BB_{\SS^q}(\by,\delta)$ with some $\by\in A$, consider the set $I_A:=\{\x\cdot\by':\by'\in A\}$. Since $\max\limits_{\by'\in A}|\x\cdot\by-\x\cdot\by'|\leq|\by-\by'|\leq\pi\delta$, we have $I_A\subset[\x\cdot\by-\pi\delta,\x\cdot\by+\pi\delta]$. So we can take $\delta_0\leq\pi\delta$ for $I_A$. Also, for any $\bz\in A$, let $\bz'\in\Exx$ be the unique nearest point on the geodesic containing both $\x$ and $\bz$. Then $\rho^*(\x,\bz)+\rho^*(\bz,\bz')=\rho^*(\bz',\x)=\frac{\pi}{2}$. Hence
    \begin{equation*}
        \begin{split}
            |\x\cdot\bz|=\cos(\rho^*(\x,\bz))=\sin\left(\frac{\pi}{2}-\rho^*(\x,\bz)\right)=\sin\left(\rho^*(\bz,\bz')\right)\geq\frac{2}{\pi}\rho^*(\bz,\bz')=\rho_{\YY}(\bz,\bz')\geq\tilde\delta.
        \end{split}
    \end{equation*}
    So the distance $\mathrm{dist}(I_A,0)$ can be estimated as $\mathrm{dist}(I_A,0)=\min\limits_{z\in A}|\x\cdot\bz-0|\geq\tilde\delta$, and we can take $\delta_1\geq\tilde\delta$. Thus
    
    \begin{equation*}
      E_R(A;G_\x)\leq2^R\delta_1^{\gamma-R}\delta_0^R\leq2^R\tilde\delta^{\gamma-R}\left(\pi\delta\right)^R.
    \end{equation*}
    Consequently, we can take $u=R-\gamma$ and $|G|_{\Delta,R,R-\gamma}\leq\left(2\pi\right)^R$ in Theorem \ref{thm:main}.
\newline    
\emph{Step 5. Upper bound of $|G|_{\XX,\alpha}$ and $|G|_{\mathcal{G}}$}
\newline
    It is also easy to observe
    $$\left\|G(\x,\circ)-G(\x',\circ)\right\|\leq \gamma|\x-\x'|\leq r\rho^*(\x,\x')\leq \pi \kappa_\SS \gamma\rho_{\XX}(\x,\x').$$
    Hence we take $|G|_{\SS^q,1}\leq\pi \kappa_\SS \gamma$. Thus $|G|_{\mathcal{G}}$ can be bounded as
    $$|G|_{\mathcal{G}}=\max\left\{|G|_{\XX,\alpha},|G|_r,|G|_{\Delta,R,u}\right\}\leq \kappa_\SS (2\pi)^R.$$
\emph{Step 6. Choice of $s$ and upper bound of $\Theta_{\tau,G}$}
\newline
    Now consider $\Theta_{\tau,G}$ and $s$. By definition, $\SS^q$ can be covered by $\kappa_\SS q^{3/2}\log q\epsilon^{-q}$ balls $\BB_{\SS^q}(\x,\epsilon)$ with each point on $\SS^q$ be overlapped at most $\kappa_\SS q\log(1+q)$ times. Let $\{\BB_{\SS^q}(\bz_j,\epsilon)\}_{j=1}^K$ be the collection of those balls satisfying $\BB_{\SS^q}(\bz_j,\epsilon)\cap\BB_{\SS^q}(\Exx,\epsilon)\neq\varnothing$, then
    $$\BB_{\SS^q}(\bz_j,\epsilon)\subset\BB_{\SS^q}(\Exx,3\epsilon),\quad j=1,\dots,K.$$
    We have $\Exx=\{\by\in\SS^q:\x\cdot\by=0\}$. Then using the formula \cite[(7.31)]{mhaskar2023local}, the $\mu^*$-measure of $\BB_{\SS^q}(\Exx,3\epsilon)$ satisfies
    \begin{equation}\label{eqn:Ex_sphere}
        \mu^*(\BB_{\SS^q}(\Exx,3\epsilon))\leq\frac{\nu_{q-1}}{\nu_q}\int_{\frac{\pi}{2}-3\frac{\pi}{2}\epsilon}^{\frac{\pi}{2}+3\frac{\pi}{2}\epsilon}\sin^{q-1}\theta d\theta\leq3\sqrt{\pi(q+2)}\epsilon,
    \end{equation}
    where $\nu_n$ is the measure of the $n$-dimensional unit sphere for $n\in\NN$. Since each point can belong to at most
    $\kappa_\SS q\log(1+q)$
    of these balls (cf. Proposition \ref{prop:coversphere}), we have
    $$K\mu^*(\BB_{\SS^q}(\bz_1,\epsilon))=\sum\limits_{j=1}^K\mu^*(\BB_{\SS^q}(\bz_j,\epsilon))\leq \kappa_\SS q\log(1+q)\mu^*(\BB_{\SS^q}(\Exx,3\epsilon))\leq 3\kappa_\SS q\sqrt{\pi(q+2)}\log(1+q)\epsilon.$$
    Thus
    \begin{equation*}
        K\leq\frac{3\kappa_\SS q\sqrt{\pi(q+2)}\log(1+q)}{\mu^*(\BB_{\SS^q}(\bz_1,\epsilon))}\epsilon=\frac{3\kappa_\SS q\sqrt{\pi(q+2)}\log(1+q)}{\mu^*\left(\BB_{\SS^q,\rho^*}\left(\bz_1,\frac{\pi}{2}\epsilon\right)\right)}\epsilon.
    \end{equation*}
    Since $\BB_{\SS^q}(\bz_1,\epsilon),\dots,\BB_{\SS^q}(\bz_K,\epsilon)$ are all the balls that intersects with $\BB_{\SS^q}(\Exx,\epsilon)$, we get from \eqref{eqn:cond_Ex_sphere} that
    \begin{equation*}
        \begin{split}
            |\tau|\left(\BB_{\SS^q}(\Exx,\epsilon)\right)\leq&\sum\limits_{j=1}^K|\tau|\left(\BB_{\SS^q}(\bz_j,\epsilon)\right)=\sum\limits_{j=1}^K|\tau|\left(\BB_{\SS^q,\rho^*}\left(\bz_j,\frac{\pi}{2}\epsilon\right)\right)\leq\sum\limits_{j=1}^K\Xi_\tau|\tau|_{TV}\mu^*\left(\BB_{\SS^q,\rho^*}\left(\bz_j,\frac{\pi}{2}\epsilon\right)\right)\\
            =&\Xi_\tau|\tau|_{TV}K\mu^*\left(\BB_{\SS^q,\rho^*}\left(\bz_1,\frac{\pi}{2}\epsilon\right)\right)\leq 3\kappa_\SS \Xi_\tau|\tau|_{TV}q\sqrt{\pi(q+2)}\log(1+q)
        \end{split}
    \end{equation*}
    Thus we get $s=q-1$ and
    \begin{equation}\label{eqn:thetaGest_thm4.1}
        \Theta_{\tau,G}\leq 3\kappa_\SS \Xi_\tau q\sqrt{\pi(q+2)}\log(1+q)\leq3\pi  \Xi_\tau\kappa_\SS q^{3/2}\log q.
    \end{equation}
\newline
    \emph{Step 7. Substitute the values above and get the conclusion}
    \newline
    Now we can apply \eqref{eqn:main1}. We note that $\alpha=1$, $q(q+2R-2ua)\le q(q+2R)\le (q+2R)^2$, and use the rest of the values as computed so far to conclude there exist $\{\by_1,\dots,\by_N\}\subset\SS^q$ and numbers $a_1,\dots,a_N$ such that
    $$\left\|\int_{\SS^q}G(\x,\by)d\tau(\by)-\sum\limits_{k=1}^Na_kG(\circ,\by_k)\right\|_{\SS^Q}\leq c_2|\tau|_{TV}\frac{\sqrt{1+\log N}}{N^{\frac{1}{2}+\frac{\gamma}{q}+\frac{\lambda}{2q}}},$$
    where $\lambda=\frac{2R-2\gamma}{2R-2\gamma+1}$ and
    \begin{equation*}
    \begin{split}
        c_2=16\sqrt{\pi}e\kappa_\SS (2\pi)^R\left[\kappa_\SS q^{3/2}(3(q+1)^{R+1}+6)\log q\right]^{\frac{1}{2}+\frac{R}{q}}\left[(q+2R+\log(\kappa_\SS q^2))(6\Xi_\tau)^{1/2}+1\right].
    \end{split}
    \end{equation*}
    This proves \eqref{eqn:rateReLUr1} in the case that $q=Q$.
\qed

\subsubsection{Proof for $q<Q$}

In this case, the terms $r,\ |G|_r,\  D_k,\ \mfc_\XX,\ \mfc_T$, and $|G|_{\XX,\alpha}$ are the same as those in the case \ref{subsec:q=Q,gammanotN}. 
The main difference here is $\Exx$ is no longer a equator of $\SS^q$. In fact, $G_\x$ is only $\gamma$-smooth for any
$$\x\in\left\{\x\in\SS^Q:\ x_1=\dots=x_{q+1}=0\right\}.$$

Therefore, for simplicity, we just choose $\Exx=\varnothing$ for all $\x\in\SS^Q$. Then by choosing $R=r=\gamma$, $|G|_{\Delta,r,0}=|G|_r=\pi^\gamma$. So we can use
    $$|G|_{\mathcal{G}}\leq \kappa_\SS \pi^\gamma.$$
Apply Theorem \ref{thm:main2} and substitute those values, we conclude there exists $\{\by_1,\dots,\by_N\}\subset T$ and numbers $a_1,\dots,a_N$ with $\sum\limits_{k=1}^N|a_k|\leq|\tau|_{TV}$ such that
\begin{equation*}
    \left\|f-\sum\limits_{k=1}^Na_kG(\circ,\by_k)\right\|_\XX\leq c_2|\tau|_{TV}\frac{\sqrt{1+\log N}}{N^{1+\frac{r}{q}}}.
\end{equation*}
where
    \begin{equation*}
    \begin{split}
        c_2=16\kappa_\SS \pi^\gamma\left[\kappa_\SS q^{3/2}(3(q+1)^\gamma+6)\log q\right]^{1+\frac{\gamma}{q}}\left[\left(\frac{Q(q+2\gamma)}{2}+\log(\kappa_\SS Q^2)\right)^{1/2}(\kappa_\SS q^{3/2}\log q)^{-1/2}+1\right].
    \end{split}
    \end{equation*}

\qed

\subsection{Proof of Theorem~\ref{thm:ReLUr2}}\label{sec:pfcor1_2}

    In this proof, $q=Q$ and $\gamma\in\NN$. In this case, all the terms are the same as those in the non-integer case \ref{subsec:q=Q,gammanotN} except $|G|_{\Delta,R,u}$ and $D_R$.

    Since $G(\x,\circ)$ is exactly a polynomial of degree $\leq \gamma$ on $\Delta(\Exx,\delta)$, we can take $\Pi_R=\Pi_\gamma$ for all $R\ge \gamma$, then $|G|_{\Delta,R,u}=0$, $D_R\leq(q+1)^\gamma$ for all $R\geq \gamma$ and $u\geq0$.
    
    In this case,
    $$|G|_{\mathcal{G}}\leq (2\pi\kappa_\SS)^\gamma.$$
    Take $u=R-\gamma$ for all $R\geq\gamma$, then
    $$R-ua=R-(R-\gamma)\frac{2R-2\gamma}{2R-2\gamma+1}=(R-\gamma)\left(1-\frac{2R-2\gamma}{2R-2\gamma+1}\right)+\gamma=\frac{R-\gamma}{2R-2\gamma+1}+\gamma\leq \gamma+1.$$
    Therefore, \eqref{eqn:main1} implies
    $$\left\|\int_{\SS^q}G(\x,\by)d\tau(\by)-\sum\limits_{k=1}^Na_kG(\circ,\by_k)\right\|_{\SS^Q}\leq c'\frac{\sqrt{1+\log N}}{N^{\frac{1}{2}+\frac{\gamma}{q}+\frac{\lambda}{2q}}},$$
    where $\lambda=\frac{2R-2\gamma}{2R-2\gamma+1}$ and
    \begin{equation*}
    \begin{split}
        c'=&16\sqrt{\pi}e\kappa_\SS(2\pi)^\gamma\left[\kappa_\SS q^{3/2}(3(q+1)^\gamma+6)\log q\right]^{1+\frac{\gamma+1}{q}}\left[(q+\gamma+1+\log(\kappa_\SS q^2))(6\Xi_\tau)^{1/2}+1\right].
    \end{split}
    \end{equation*}
    This holds for all $R\geq\gamma$, we take $R$ sufficiently large such that
    $$N^{\lambda}=N^{\frac{2R-2\gamma}{2R-2\gamma+1}}\ge N/2.$$
    Then
    $$\left\|\int_{\SS^q}G(\x,\by)d\tau(\by)-\sum\limits_{k=1}^Na_kG(\circ,\by_k)\right\|_{\SS^Q}\leq c_2'|\tau|_{TV}\frac{\sqrt{1+\log N}}{N^{\frac{1}{2}+\frac{2\gamma+1}{2q}}}.$$
  where $c_2'=2c'$ is the constant in \eqref{eqn:integer_gamma_const}.  
\qed

\subsection{Proof of Theorem \ref{thm:zonal}}\label{sec:pfcor2}

For each $\x\in\SS^q$, $\Exx=\{\x\}$. Let $\tilde\rho^*$ be the geodesic distance on $\SS^Q$, we define the distance $\rho_{\XX}$ on $\XX=\SS^Q$ as \eqref{eqn:rho2_modify}:
$$\rho_{\XX}(\x,\by)=\frac{2}{\pi }\tilde\rho^*(\x,\by),\quad\x,\by\in\SS^Q.$$
Define $\rho^*$ be the geodesic distance on $ \YY=T=\SS^q$ and denote the distance $\rho_{\YY}$ on $\SS^q$ as \eqref{eqn:rho1_modify}:
$$\rho_{\YY}(\x,\by)=\frac{2}{\pi}\rho^*(\x,\by),\quad\x,\by\in\SS^q.$$
Using the same argument as in Section \ref{sec:pfcor1}, we take
$$\mfc_\XX\leq\kappa_\SS Q^{3/2}\log Q\leq\kappa_\SS Q^2,\quad\mfc_T\leq\kappa_\SS q^{3/2}\log q.$$

We take $r=2\gamma$. Using a similar discussion as in Section \ref{sec:pfcor1}, we can bound
$$|G|_{r}\leq\pi^{2\gamma},\qquad|G|_{\SS^Q,1}\leq2\pi \kappa_\SS\gamma.$$

The number $R$ only need to satisfy the condition $R\geq r=2\gamma$, hence we can choose $R=r$, for which $u=0$ and $|G|_{\Delta,R,0}\leq|G|_{r}\leq\pi^{2\gamma}$. Thus
$$|G|_\mathcal{G}\leq \kappa_\SS\pi^{2\gamma}.$$
In this case,
$$D_R+2=\binom{q+\lfloor2\gamma+1\rfloor}{\lfloor2\gamma+1\rfloor}+2\leq(q+1)^{2\gamma+1}+2.$$
By \eqref{eqn:main2}, there exists $\{\by_1,\dots,\by_N\}\subset\SS^q$ and numbers $a_1,\dots,a_N$, such that
\begin{equation*}
    \left\|f-\sum\limits_{k=1}^Na_kG(\circ,\by_k)\right\|_{\SS^Q}\leq c_3|\tau|_{TV}\left(\frac{1+\log N}{N^{1+4\gamma/q}}\right)^{1/2},
\end{equation*}
where
    \begin{equation*}
    \begin{split}
        c_3=&8\kappa_\SS \pi^{2\gamma}\left[\kappa_\SS q^{3/2}(3(q+1)^{2\gamma+1}+6)\log q\right]^{1+\frac{2\gamma}{q}}\left[\left(\frac{Q(q+2\gamma)}{2}+\log(\kappa_\SS Q^2)\right)^{1/2}\left(\kappa_\SS q^{3/2}\log q\right)^{-\frac{1}{2}}+1\right].
    \end{split}
    \end{equation*}

\qed

\subsection{Proof of Theorem~\ref{thm:Gaussian}}\label{sec:pfcor3}

In this example, we take $\XX=B^Q$, $\YY=B^q$, $\rho_{\XX}=\rho_\YY$ be the Euclidean distance in $\RR^Q$. Then the diameter of $B^q$ and $B^Q$ is $2$ (cf. Remark \ref{rem:cover_dist}). We have $\Exx=\{\x\}.$

    Let $\x\in\XX$, $\by\in\YY$. Since $G_\x=\exp(-|\x-\circ|)$ is a Lipschitz function, we can take $\alpha=1$, and for any $\by'\in\BB_{\YY}(\by,\delta)$,
    $$|G_\x(\by)-G_\x(\by')|\leq \biggl|-|\x-\by|+|\x-\by'|\biggr|\leq|\by-\by'|\leq \delta.$$
    Then
    $$E_1(\BB_\YY(\by,\delta),G_\x)\leq\sup\limits_{\by'\in\BB_{\YY}(\by,\delta)}|G_\x(\by')-G_\x(\by)|\leq \delta.$$
    Taking $R=r=1$, this implies $|G|_1\leq 1$. It is a trivial fact that $u=0$, and $|G|_{\Delta,1,0}\leq|G|_1\leq 1$.

    Using the same argument, we have
    $$|G(\x,\by)-G(\x',\by)|\leq\biggl|-|\x-\by|+|\x'-\by|\biggr|\leq|\x-\x'|\leq\rho_\XX(\x,\x').$$

    Then $|G|_{\XX,1}\leq 1$. So we have
    $$|G|_{\mathcal{G}}\leq1.$$
    Again, we have
    $$D_1+2\leq(q+1)+2=q+3.$$

    Notice that covering a ball of radius $1$ by balls of radius $\epsilon$ is equivalent to covering a ball of radius $(2\epsilon)^{-1}$ by balls of radius $1/2$, we can conclude by \cite[Theorem 3.1]{verger2005covering} that there exists an absolute constant $\kappa_B$ such that
    $$N(\YY,\epsilon)\leq (\kappa_Bq^{3/2}\log q)\epsilon^{-q},\quad N(\XX,\epsilon)\leq(\kappa_BQ^{3/2}\log Q)\epsilon^{-Q}.$$
    This implies $\mfc_T\leq\kappa_Bq^{3/2}\log q$ and $\mfc_\XX\leq\kappa_BQ^{3/2}\log Q$.

    Applying \eqref{eqn:main2}, we conclude for any $N\geq3(q+1)+6$, there exists $\{\by_1,\dots,\by_N\}\subset B^q$ and numbers $a_1,\dots,a_N$, such that
\begin{equation*}
    \left\|f-\sum\limits_{k=1}^Na_kG(\circ,\by_k)\right\|_{B^Q}\leq c_4|\tau|_{TV}\left(\frac{1+\log N}{N^{1+2/q}}\right)^{1/2},
\end{equation*}
where
    \begin{equation*}
    \begin{split}
        c_4=8\left[(\kappa_Bq^{3/2}\log q)(3q+9)\right]^{1+\frac{1}{q}}\left[\left(\frac{q+2}{2}Q+2\log Q+\log\kappa_B\right)^{1/2}\left(\kappa_Bq^{3/2}\log q\right)^{-1/2}+1\right].
    \end{split}
    \end{equation*}

\qed



\begin{thebibliography}{10}

\bibitem{aronszajn1950theory}
N.~Aronszajn.
\newblock Theory of reproducing kernels.
\newblock {\em Transactions of the American mathematical society},
  68(3):337--404, 1950.

\bibitem{Barron1993}
A.~R. Barron.
\newblock Universal approximation bounds for superpositions of a sigmoidal
  function.
\newblock {\em Information Theory, IEEE Transactions on}, 39(3):930--945, 1993.

\bibitem{bartolucci2021understanding}
F.~Bartolucci, E.~De~Vito, L.~Rosasco, and S.~Vigogna.
\newblock Understanding neural networks with reproducing kernel banach spaces.
\newblock {\em arXiv preprint arXiv:2109.09710}, 2021.

\bibitem{belkin2004regularization}
M.~Belkin, I.~Matveeva, and P.~Niyogi.
\newblock Regularization and semi-supervised learning on large graphs.
\newblock In {\em Learning Theory: 17th Annual Conference on Learning Theory,
  COLT 2004, Banff, Canada, July 1-4, 2004. Proceedings 17}, pages 624--638.
  Springer, 2004.

\bibitem{belkin2004semi}
M.~Belkin and P.~Niyogi.
\newblock Semi-supervised learning on riemannian manifolds.
\newblock {\em Machine learning}, 56:209--239, 2004.

\bibitem{belkin2006convergence}
M.~Belkin and P.~Niyogi.
\newblock Convergence of laplacian eigenmaps.
\newblock {\em Advances in neural information processing systems}, 19, 2006.

\bibitem{boroczky2003covering}
K.~B{\"o}r{\"o}czky and G.~Wintsche.
\newblock Covering the sphere by equal spherical balls.
\newblock {\em Discrete and Computational Geometry: The Goodman-Pollack
  Festschrift}, pages 235--251, 2003.

\bibitem{bourgain1988distribution}
J.~Bourgain and J.~Lindenstrauss.
\newblock Distribution of points on spheres and approximation by zonotopes.
\newblock {\em Israel Journal of Mathematics}, 64(1):25--31, 1988.

\bibitem{devore1989optimal}
R.~A. DeVore, R.~Howard, and C.~A. Micchelli.
\newblock Optimal nonlinear approximation.
\newblock {\em Manuscripta mathematica}, 63(4):469--478, 1989.

\bibitem{devore1996some}
R.~A. DeVore and V.~N. Temlyakov.
\newblock Some remarks on greedy algorithms.
\newblock {\em Advances in computational Mathematics}, 5(1):173--187, 1996.

\bibitem{dick2010digital}
J.~Dick and F.~Pillichshammer.
\newblock {\em Digital nets and sequences: discrepancy theory and quasi--Monte
  Carlo integration}.
\newblock Cambridge University Press, 2010.

\bibitem{dung2021deep}
D.~D{\~u}ng and v.~K. Nguyeb=n.
\newblock Deep relu neural networks in high-dimensional approximation.
\newblock {\em Neural Networks}, 142:619--635, 2021.

\bibitem{compbio}
M.~Ehler, F.~Filbir, and H.~N. Mhaskar.
\newblock Locally learning biomedical data using diffusion frames.
\newblock {\em Journal of Computational Biology}, 19(11):1251--1264, 2012.

\bibitem{feng2023radial}
H.~Feng, S.-B. Lin, and D.-X. Zhou.
\newblock Radial basis function approximation with distributively stored data
  on spheres.
\newblock {\em Constructive Approximation}, pages 1--31, 2023.

\bibitem{modlpmz}
F.~Filbir and H.~N. Mhaskar.
\newblock Marcinkiewicz--{Z}ygmund measures on manifolds.
\newblock {\em Journal of Complexity}, 27(6):568--596, 2011.

\bibitem{gavish2010multiscale}
M.~Gavish, B.~Nadler, and R.~R. Coifman.
\newblock Multiscale wavelets on trees, graphs and high dimensional data:
  Theory and applications to semi supervised learning.
\newblock In {\em Proceedings of the 27th International Conference on Machine
  Learning (ICML-10)}, pages 367--374, 2010.

\bibitem{hackbusch2012tensor}
W.~Hackbusch.
\newblock {\em Tensor spaces and numerical tensor calculus}, volume~42.
\newblock Springer Science \& Business Media, 2012.

\bibitem{hangelbroek2012polyharmonic}
T.~Hangelbroek, F.~J. Narcowich, and J.~D. Ward.
\newblock Polyharmonic and related kernels on manifolds: interpolation and
  approximation.
\newblock {\em Foundations of Computational Mathematics}, 12:625--670, 2012.

\bibitem{Barron2018}
J.~M. Klusowski and A.~R. Barron.
\newblock Approximation by combinations of relu and squared relu ridge
  functions with $\ell^1$ and $\ell^0$ controls.
\newblock {\em IEEE Transactions on Information Theory}, 64(12):7649--7656,
  2018.

\bibitem{kurkova1}
V.~K\r{u}rkov{\'a} and M.~Sanguineti.
\newblock Bounds on rates of variable basis and neural network approximation.
\newblock {\em IEEE Transactions on Information Theory}, 47(6):2659--2665,
  2001.

\bibitem{kurkova2}
V.~K\r{u}rkov{\'a} and M.~Sanguineti.
\newblock Comparison of worst case errors in linear and neural network
  approximation.
\newblock {\em IEEE Transactions on Information Theory}, 48(1):264--275, 2002.

\bibitem{le2006continuous}
Q.~T. Le~Gia, F.~J. Narcowich, J.~D. Ward, and H.~Wendland.
\newblock Continuous and discrete least-squares approximation by radial basis
  functions on spheres.
\newblock {\em Journal of Approximation Theory}, 143(1):124--133, 2006.

\bibitem{ma2022uniform}
L.~Ma, J.~W. Siegel, and J.~Xu.
\newblock Uniform approximation rates and metric entropy of shallow neural
  networks.
\newblock {\em Research in the Mathematical Sciences}, 9(3):46, 2022.

\bibitem{mauropap}
M.~Maggioni and H.~N. Mhaskar.
\newblock Diffusion polynomial frames on metric measure spaces.
\newblock {\em Applied and Computational Harmonic Analysis}, 24(3):329--353,
  2008.

\bibitem{makovoz1998uniform}
Y.~Makovoz.
\newblock Uniform approximation by neural networks.
\newblock {\em Journal of Approximation Theory}, 95(2):215--228, 1998.

\bibitem{mao2022approximation}
T.~Mao and D.-X. Zhou.
\newblock Approximation of functions from korobov spaces by deep convolutional
  neural networks.
\newblock {\em Advances in Computational Mathematics}, 48(6):84, 2022.

\bibitem{data_based_construction2019}
H.~Mhaskar, S.~V. Pereverzyev, V.~Y. Semenov, and E.~V. Semenova.
\newblock Data based construction of kernels for semi-supervised learning with
  less labels.
\newblock {\em Frontiers in Applied Mathematics and Statistics}, 5:21, 2019.

\bibitem{tractable}
H.~N. Mhaskar.
\newblock On the tractability of multivariate integration and approximation by
  neural networks.
\newblock {\em Journal of Complexity}, 20(4):561--590, 2004.

\bibitem{zfquadpap}
H.~N. Mhaskar.
\newblock Weighted quadrature formulas and approximation by zonal function
  networks on the sphere.
\newblock {\em Journal of Complexity}, 22(3):348--370, 2006.

\bibitem{eignet}
H.~N. Mhaskar.
\newblock Eignets for function approximation on manifolds.
\newblock {\em Applied and Computational Harmonic Analysis}, 29(1):63--87,
  2010.

\bibitem{heatkernframe}
H.~N. Mhaskar.
\newblock A generalized diffusion frame for parsimonious representation of
  functions on data defined manifolds.
\newblock {\em Neural Networks}, 24(4):345--359, 2011.

\bibitem{mhaskar2020dimension}
H.~N. Mhaskar.
\newblock Dimension independent bounds for general shallow networks.
\newblock {\em Neural Networks}, 123:142--152, 2020.

\bibitem{mhaskar2020kernel}
H.~N. Mhaskar.
\newblock Kernel-based analysis of massive data.
\newblock {\em Frontiers in Applied Mathematics and Statistics}, 6:30, 2020.

\bibitem{mhaskar2023local}
H.~N. Mhaskar.
\newblock Local approximation of operators.
\newblock {\em Applied and Computational Harmonic Analysis}, 2023.

\bibitem{sphrelu}
H.~N. Mhaskar.
\newblock Function approximation with zonal function networks with activation
  functions analogous to the rectified linear unit functions.
\newblock {\em Journal of Complexity}, 51:1--19, April 2019.

\bibitem{dimindbd}
H.~N. Mhaskar and C.~A. Micchelli.
\newblock Dimension-independent bounds on the degree of approximation by neural
  networks.
\newblock {\em IBM Journal of Research and Development}, 38(3):277--284, 1994.

\bibitem{mhas_sergei_maryke_diabetes2017}
H.~N. Mhaskar, S.~V. Pereverzyev, and M.~D. van~der Walt.
\newblock A deep learning approach to diabetic blood glucose prediction.
\newblock {\em Frontiers in Applied Mathematics and Statistics}, 3:14, 2017.

\bibitem{dingxuanpap}
H.~N. Mhaskar and T.~Poggio.
\newblock Deep vs. shallow networks: An approximation theory perspective.
\newblock {\em Analysis and Applications}, 14(06):829--848, 2016.

\bibitem{montanelli2019new}
H.~Montanelli and Q.~Du.
\newblock New error bounds for deep relu networks using sparse grids.
\newblock {\em SIAM Journal on Mathematics of Data Science}, 1(1):78--92, 2019.

\bibitem{novak2008tractability}
E.~Novak and H.~Wo{\'z}niakowski.
\newblock {\em Tractability of Multivariate Problems: Standard information for
  functionals}, volume~12.
\newblock European Mathematical Society, 2008.

\bibitem{pisier1981remarques}
G.~Pisier.
\newblock Remarques sur un r{\'e}sultat non publi{\'e} de b. maurey.
\newblock {\em S{\'e}minaire d'Analyse fonctionnelle (dit" Maurey-Schwartz")},
  pages 1--12, 1981.

\bibitem{pollard2012bk}
D.~Pollard.
\newblock {\em Convergence of stochastic processes}.
\newblock Springer Science \& Business Media, 2012.

\bibitem{schaback1999native}
R.~Schaback.
\newblock Native hilbert spaces for radial basis functions i.
\newblock In {\em New Developments in Approximation Theory: 2nd International
  Dortmund Meeting (IDoMAT)’98, Germany, February 23--27, 1998}, pages
  255--282. Springer, 1999.

\bibitem{siegel2020approximation}
J.~W. Siegel and J.~Xu.
\newblock Approximation rates for neural networks with general activation
  functions.
\newblock {\em Neural Networks}, 128:313--321, 2020.

\bibitem{siegel2022high}
J.~W. Siegel and J.~Xu.
\newblock High-order approximation rates for shallow neural networks with
  cosine and reluk activation functions.
\newblock {\em Applied and Computational Harmonic Analysis}, 58:1--26, 2022.

\bibitem{siegel2022sharp}
J.~W. Siegel and J.~Xu.
\newblock Sharp bounds on the approximation rates, metric entropy, and n-widths
  of shallow neural networks.
\newblock {\em Foundations of Computational Mathematics}, pages 1--57, 2022.

\bibitem{song2011reproducing}
G.~Song and H.~Zhang.
\newblock Reproducing kernel banach spaces with the $\ell^1$ norm ii: Error
  analysis for regularized least square regression.
\newblock {\em Neural computation}, 23(10):2713--2729, 2011.

\bibitem{song2013reproducing}
G.~Song, H.~Zhang, and F.~J. Hickernell.
\newblock Reproducing kernel banach spaces with the $\ell^1$ norm.
\newblock {\em Applied and Computational Harmonic Analysis}, 34(1):96--116,
  2013.

\bibitem{suzuki2018adaptivity}
T.~Suzuki.
\newblock Adaptivity of deep relu network for learning in besov and mixed
  smooth besov spaces: optimal rate and curse of dimensionality.
\newblock {\em arXiv preprint arXiv:1810.08033}, 2018.

\bibitem{temlyakov1986approximation}
V.~N. Temlyakov.
\newblock Approximation of functions with bounded mixed derivative.
\newblock {\em Trudy Mat. Inst. Steklov}, 178(1):112, 1986.

\bibitem{verger2005covering}
J.~L.~Verger-Gaugry.
\newblock Covering a ball with smaller equal balls in $\RR^n$.
\newblock {\em Discrete $\&$ Computational Geometry}, 38: 143--155, 2005.

\bibitem{xu2020finite}
J.~Xu.
\newblock The finite neuron method and convergence analysis.
\newblock {\em arXiv preprint arXiv:2010.01458}, 2020.

\bibitem{yarotsky2018optimal}
D.~Yarotsky.
\newblock Optimal approximation of continuous functions by very deep {ReLU}
  networks.
\newblock {\em arXiv preprint arXiv:1802.03620}, 2018.

\bibitem{zhang2016sampling}
X.~Zhang, L.~Zong, Q.~You, and X.~Yong.
\newblock Sampling for nystr{\"o}m extension-based spectral clustering:
  Incremental perspective and novel analysis.
\newblock {\em ACM Transactions on Knowledge Discovery from Data (TKDD)},
  11(1):1--25, 2016.

\end{thebibliography}

\end{document}